\documentclass[12pt]{article}
\usepackage{times}

\usepackage{amsmath,amsfonts,bm}
\usepackage{amsthm,amssymb}
\usepackage{mathtools}
\usepackage{xcolor}

\DeclareMathOperator*{\argmin}{arg\,min}
\newcommand{\eg}{{\emph{e.g.}}}
\newcommand{\cA}{\mathcal{A}}

\newcommand{\cD}{\mathcal{D}}

\newcommand{\cL}{\mathcal{L}}

\newcommand{\cN}{\mathcal{N}}
\newcommand{\cO}{\mathcal{O}}
\newcommand{\cP}{\mathcal{P}}

\newcommand{\cU}{\mathcal{U}}

\newcommand{\bE}{\mathbb{E}}

\newcommand{\bR}{\mathbb{R}}

\newcommand{\X}{{X}}
\newcommand{\Y}{{Y}}

\newcommand{\vw}{{\boldsymbol{w}}}
\newcommand{\vv}{{\boldsymbol{v}}}
\newcommand{\homo}{H}

\usepackage{url}
\usepackage{fancyhdr}
\usepackage{booktabs}
\usepackage{nicefrac}
\usepackage{microtype}
\usepackage[round]{natbib}
\usepackage{enumitem}
\usepackage{graphicx}
\usepackage{algorithmic}
\usepackage{bbm}
\usepackage{diagbox}
\usepackage{float}
\usepackage{lipsum}
\usepackage{multirow}
\usepackage{caption}
\usepackage{wrapfig}
\usepackage{selectp}
\renewcommand{\bibname}{References}
\renewcommand{\bibsection}{\subsubsection*{\bibname}}
\usepackage{eso-pic}
\usepackage{forloop}
\usepackage[bb=pazo]{mathalpha}
\usepackage{subcaption}
\newtheorem*{theorem*}{Theorem} 
\newtheorem{theorem}{Theorem}
\newtheorem{corollary}{Corollary}
\newtheorem{lemma}{Lemma}
\newtheorem{definition}{Definition}
\newtheorem{assumption}{Assumption}

\newtheorem{proposition}{Proposition}

\theoremstyle{definition} 

\newtheorem{example}[theorem]{Example}

\usepackage{titletoc}
\usepackage{appendix}
\usepackage[most]{tcolorbox}
\usepackage[align=center,nobreak=true,framemethod=tikz,skipabove=2pt,skipbelow=1pt,innertopmargin=-3pt,innerbottommargin=3pt,innerleftmargin=5pt,innerrightmargin=5pt,leftmargin=-2pt,rightmargin=-2pt,tikzsetting={
    draw=black, 
    line width=1pt}]{mdframed}
\usepackage{thm-restate}
\definecolor{mygrey}{rgb}{0.8, 0.8, 0.8}
\definecolor{darkgrey}{rgb}{0.5, 0.5, 0.5}
\usepackage{makecell}
\usepackage{xspace}
\usepackage{tabularx}       
\usepackage{tikz}           
\usepackage[colorlinks=true, linkcolor=blue, citecolor=blue]{hyperref}

\newcounter{observationcounter}
\makeatletter
\newcommand{\obsref}[1]{%
    \@ifundefined{obs@#1}{%
        ??%
    }{%
        \hyperref[#1]{\csname obs@#1\endcsname}%
    }%
}

\makeatother

\usepackage{parskip}
\title{Generalization Bounds of Stochastic Gradient Descent in Homogeneous Neural Networks}

\author{
  \small
  Wenquan Ma
  \and
  \small
  Yang Sui
  \and
  \small
  Jiaye Teng\thanks{Author names are in alphabetical order. Correspondence to: Jiaye Teng \texttt{<tengjiaye@sufe.edu.cn>}}
  \and
  \small
  Bohan Wang
  \and
  \small
  Jing Xu
  \and
  \small
  Jingqin Yang
}
\date{}

\ifdefined\usebigfont

\usepackage{times}
\usepackage[fontsize=13pt]{scrextend}
\usepackage[left=1.56in,right=1.56in,top=1.6in,bottom=1.74in]{geometry}

\else
\usepackage[margin=1in]{geometry}
\fi

\begin{document}

\maketitle

\begin{abstract}
Algorithmic stability is among the most potent techniques in generalization analysis.
However, its derivation usually requires a stepsize $\eta_t = \cO(1/t)$ under non-convex training regimes, where $t$ denotes iterations.
This rigid decay of the stepsize potentially impedes optimization and may not align with practical scenarios.
In this paper, we derive the generalization bounds under the homogeneous neural network regimes, proving that this regime enables slower stepsize decay of order $\Omega(1/\sqrt{t})$ under mild assumptions.
We further extend the theoretical results from several aspects, \eg, non-Lipschitz regimes. 
This finding is broadly applicable, as homogeneous neural networks encompass fully-connected and convolutional neural networks with ReLU and LeakyReLU activations.
\end{abstract}

\section{Introduction}\label{sec: intro}

Stochastic gradient descent (SGD) has become a cornerstone in the field of deep learning. Extensive empirical applications have demonstrated its remarkable success in both optimization and generalization~\citep{DBLP:journals/corr/GoyalDGNWKTJH17, DBLP:conf/nips/BrownMRSKDNSSAA20}. On the theoretical side, researchers have proposed various approaches to address the challenge of generalization~\citep{DBLP:conf/colt/McAllester99,DBLP:conf/aistats/RussoZ16,bartlett2020benign}. Among these, one of the most popular is uniform convergence~\citep{DBLP:conf/nips/BartlettFT17,DBLP:conf/iclr/WeiM20}, which, unfortunately, has proven to be inadequate for understanding generalization in deep learning~\citep{shalev2010learnability,DBLP:conf/nips/NagarajanK19,DBLP:conf/iclr/GlasgowWW023}.

As an alternative to uniform convergence, algorithmic stability has emerged as a potential solution. In convex training scenarios, it provides guaranteed generalization for a reasonable number of training iterations~\citep{DBLP:journals/jmlr/BousquetE02, DBLP:conf/icml/HardtRS16}. However, real-world applications typically occur in non-convex landscapes, where the generalization behavior of SGD remains elusive. Existing works on algorithmic stability under non-convex training regimes usually apply to the case with a stepsize of $\cO(1/t)$, where $t$ denotes the training iteration~\citep{DBLP:conf/icml/HardtRS16, DBLP:conf/icml/KuzborskijL18, DBLP:conf/uai/ZhangZBP0022}. Unfortunately, this stepsize is rarely used in practice due to its slow training speed (see Section \ref{sec: optimization} for theoretical insights and Figure~\ref{fig:cifar10} for illustration). This creates a gap between theory and practice, significantly reducing the practical utility of algorithmic stability.

\begin{figure*}[t]  
 \centering
    \subfloat[Training accuracy]{\includegraphics[width=0.35\textwidth]{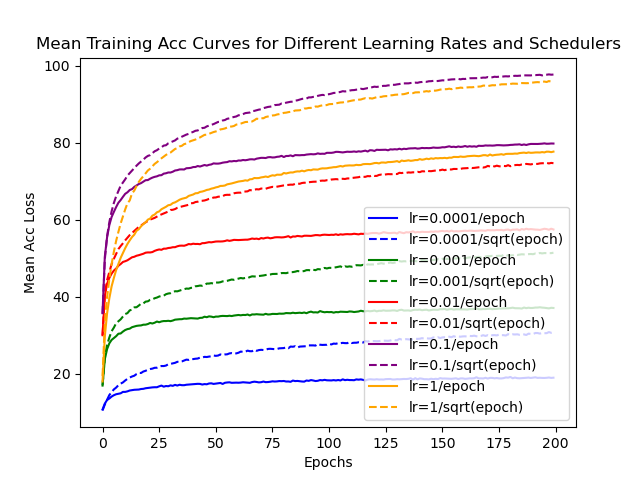}}
    \hspace{4em}
    \subfloat[Test Accuracy]{\includegraphics[width=0.35\textwidth]{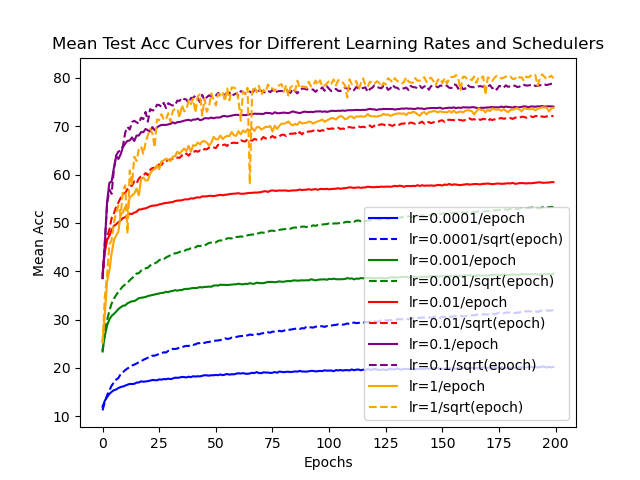}}
  \caption{The training and test accuracy curves for SGD with various stepsize schedulers on CIFAR-10 using ResNet 18. Models trained with stepsize $\Theta(1/\sqrt{t})$ (dotted line) converge faster and perform better than those trained with stepsize $\Theta(1/t)$ (solid line).}  
  \label{fig:cifar10}  
\end{figure*}

To bridge this gap, there are three main approaches: (a) employing alternative algorithms such as Stochastic Gradient Langevin Dynamics (SGLD) instead of SGD~\citep{DBLP:conf/colt/Mou0Z018,DBLP:conf/iclr/LiLQ20,DBLP:conf/nips/FarghlyR21,DBLP:conf/nips/BassilyGM21}; (b) utilizing Polyak-Lojasiewicz (PL) conditions, which are closely associated with convexity~\citep{DBLP:conf/icml/CharlesP18,DBLP:conf/iclr/NikolakakisHKK23,zhu2024stabilitysharperriskbounds}; (c) adopting alternative generalization metrics such as the gradient norm rather than the traditional generalization gap~\citep{lei2023stability,zhang2024generalization}.
However, 
(a) is challenging to extend to SGD settings;
(b) relies on a restrictive assumption that is difficult to relax;
and (c) typically fails to generalize to multi-pass scenarios.
Consequently, it remains under-explored how SGD effectively generalizes in multi-epoch regimes without convexity-related assumptions beyond $\cO(1/t)$ stepsize.

In this paper, we take a step forward in bridging the gap by proving that SGD can generalize with stepsizes of order $\Omega(1/\sqrt{t})$. 
To the best of our knowledge, this is the first result that relaxes the stepsize requirement while allowing multiple uses of each sample under non-convex SGD training regimes without relying on additional assumptions like PL conditions. 
The derivation of this result relies on a notion called \emph{homogeneous} neural networks~\citep{DBLP:conf/nips/DuHL18,DBLP:conf/iclr/LyuL20}. 
Informally, a neural network $\Phi$ is $H$-homogeneous if
\begin{equation*}
    \forall c>0: \Phi(c\vw; \X) = c^H\Phi(\vw; \X) \ \text{for} \ \text{all} \ \vw \ \text{and} \ \X,
\end{equation*}
where $\vw$ denotes the weight and $\X$ denotes the input. 
Our paper proves that for homogeneous neural networks with
a loss function that preserves homogeneity, if the loss is Lipschitz and approximately smooth, its generalization gap can be bounded over a reasonable number of training iterations. 

\subsection{Our Results in More Detail.}
In particular, we consider a binary classification task with noisy labels $\{-1,1\}$, leading to a non-zero Bayesian optimal within the sphere. This condition is imposed to regulate the decay rate of the weight norm. To harness the homogeneity of $H$-homogeneous neural networks defined above, we focus on a loss function of the form $\ell(y, \hat{y}) = \max\{0, -y\hat{y}\}$.
Despite the fact that the loss function might encourage a small prediction $\hat{y}$ and correspondingly small weights, we concentrate on the direction of the weights rather than their absolute values, which is valid in classification regimes. 
To simplify our discussion, we define all generalization metrics within the sphere, as shown in \eqref{eqn:lossevaluation}.
Besides homogeneity, we anchor our analysis on two mild assumptions:

\begin{itemize}[nosep]
    \item \textbf{Bounded Loss} (Assumption~\ref{assump: loss bound}) We assume that the individual loss is upper bounded, and the whole training loss is lower bounded by its Bayesian optimal value. This assumption is necessary because, without it, achieving a consistent generalization gap in the presence of noisy labels would be unattainable.
    \item \textbf{Lipschitz and Approximately Smooth} (Assumption~\ref{assump: approximate smooth and lipshitz}) This assumption is standard in algorithmic stability analysis. We extend the previous smoothness assumption to approximate smoothness, which aligns better with ReLU activations.
\end{itemize}

We derive Theorem~\ref{thm: homo leads to generalization} based on the above assumptions, elucidating the correlation between homogeneous neural networks and generalization: 
\begin{equation*}
\bE_{\cA,\cD} \cL^s(\vw_T) - \hat{\cL}^s(\vw_T) \leq \left[L m_2\right]^{\frac{1}{m_1 + 1}}\left[1 + \frac{1}{m_1}\right] \frac{T^\frac{m_1}{m_1 + 1}}{n},
\end{equation*}
where $T$ is the number of iterations, $n$ is the sample size. 
We use $\cA$ and $\cD$ to represent the randomized algorithm and the distribution of training dataset, respectively. 
A comprehensive summary of the notation used throughout this paper is provided in Appendix~\ref{app:notations}. 
This bound is derived under stepsizes satisfying $\bE \eta_t = \Omega(t^{\frac{H-4}{2}})$ and applies to the multi-pass training regime. 
For the specific case of $H = 3$, the required stepsize simplifies to $\bE \eta_t = \Omega(1/\sqrt{t})$. 
Notably, these homogeneous neural networks can be constructed by normalizing multiple layers while leaving the final several layers unnormalized (see Proposition~\ref{prop: homo-property}). 
As such, our results encompass a wide range of scenarios. 

\textbf{Extending the Reach of Our Findings.} Our findings can be extended in several directions. 
\emph{First}, we demonstrate the broad applicability of our framework.  While our main result focuses on the loss function $\ell(y, \hat{y}) = \max\{0, -y\hat{y}\}$, Corollary~\ref{cor: extension of loss} extends our findings to other loss functions, including smooth homogeneous variants, \eg, $\ell(y, \hat{y}) = \max\{0, -y\hat{y}\}^u$ with $u \geq 1$, and functions with relaxed homogeneity. 
Similarly, we expand the scope of activation functions beyond ReLU, for instance $\sigma(\cdot) = \max\{0, \cdot\}^u$, and architectures beyond linear connections, such as normalized ResNet models, as described in Proposition~\ref{prop: homo-property}. 
\emph{Second}, for the specific case of three-homogeneous networks, Corollary~\ref{cor: Three-Homogeneous Leads to Generalization} demonstrates that the test error consistently converges to the Bayesian optimal if the training loss approximates the Bayesian optimal within a reasonable number of iterations. 
\emph{Third}, Theorem~\ref{thm: separation} shows a separation in optimization performance, indicating that our proposed stepsizes of $\bE \eta_t=\Omega(1 / \sqrt{t})$ achieve polynomial convergence, potentially outperforming the logarithmic convergence of $\bE \eta_t=\cO(1 / t)$. 
Example~\ref{ex: compatibility_2d} further illustrates the compatibility of our setting regarding generalization and optimization. 
\emph{Finally}, Section~\ref{sec:relax-lipschitz} proves that our framework ensures generalization without Lipschitz assumptions via an adapted on-average model stability, while maintaining the applicability to the multi-pass training regime. 

\textbf{Proof Key Insights.}
Our key insights stem from a straightforward observation:
For classification problems with homogeneous neural networks, accuracy depends solely on the direction of the weights.
To utilize this insight, we project the training trajectory onto a sphere. 
While the observed stepsize is in order $\Omega(1/\sqrt{t})$, the effective stepsize of the projected trajectory can be in order $\Theta(1/t)$. 
This leads to generalization, as demonstrated by algorithmic stability. 
One may question whether this effective stepsize $\Theta(1/t)$ negatively affects optimization. 
To address this concern, we provide additional discussions on optimization performance in Section~\ref{sec: optimization} and present the training loss curve across different stepsize regimes in Figure~\ref{fig: loss curve}~(left), which illustrates that the training loss can be optimized efficiently with the stepsize $\Omega(1/\sqrt{t})$. 

\subsection{Related Works}
\label{sec:related_works}
\begin{table*}[t] 
\caption{Comparison of our method with related works on non-convex training regimes with SGD.}
\label{tab:comparison}
\begin{center}
\begin{small}
\begin{tabular*}{\textwidth}{@{\extracolsep{\fill}}ccccccc@{}}
\toprule
&
\begin{tabular}[c]{@{}c@{}}Stepsize\\ Order\end{tabular}&
\begin{tabular}[c]{@{}c@{}} PL \\ Condition\end{tabular} &
\begin{tabular}[c]{@{}c@{}}Multi-\\ pass\end{tabular}&
\begin{tabular}[c]{@{}c@{}}Overpara-\\meterization\end{tabular}&
\begin{tabular}[c]{@{}c@{}} Generalization \\ Metric \end{tabular} \\
\midrule
\midrule
\citet{DBLP:conf/icml/HardtRS16}   & {\color{red}$\cO(1/t)$} & No& Yes & Yes& Loss \\
\citet{DBLP:conf/icml/KuzborskijL18} & {\color{red}$\cO(1/t)$} & No & Yes & Yes& Loss \\
\citet{DBLP:conf/uai/ZhangZBP0022}  & {\color{red}$\cO(1/t)$} & No & Yes& Yes& Loss\\
\citet{DBLP:conf/icml/CharlesP18}   & $\cO(1)$  & {\color{red}Yes}& Yes & Yes& Loss\\
\citet{DBLP:conf/iclr/NikolakakisHKK23}   & $\cO(1)$  & {\color{red}Yes}& Yes & Yes& Loss\\
\citet{zhu2024stabilitysharperriskbounds}   & $\cO(1)$  & {\color{red}Yes}& Yes & Yes& Loss\\
\citet{lei2021learning}          & $o(1)$ & No & Yes& {\color{red}No}& Gradient  \\ 
\citet{lei2023stability}       & $\cO(1)$ & No& {\color{red}No} & Yes& Gradient  \\ 
\citet{zhang2024generalization} & $\cO(1)$ & No& {\color{red}No} & Yes& Gradient  \\ 
\midrule
This Paper     & $\Omega(1/\sqrt{t})$  & No& Yes& Yes& Loss \\
\bottomrule
\end{tabular*}
\end{small}
\end{center}
\end{table*}
\textbf{Algorithmic Stability under Non-convex Training. }
To derive algorithmic stability in non-convex regimes with SGD, a typical requirement is to employ stepsize $\cO(1/t)$~\citep{DBLP:conf/icml/HardtRS16,DBLP:conf/icml/KuzborskijL18,DBLP:conf/uai/ZhangZBP0022}.
To relax the stepsize requirement, one possible approach is to utilize alternative algorithms (\eg, SGLD)~\citep{DBLP:conf/colt/Mou0Z018,DBLP:conf/iclr/LiLQ20,DBLP:conf/nips/FarghlyR21,DBLP:conf/nips/BassilyGM21}.
When restricted to SGD training with $\omega(1/t)$ stepsize, a line of work leverages the PL condition, which closely relates to convexity~\citep{DBLP:conf/icml/CharlesP18,DBLP:conf/iclr/NikolakakisHKK23,zhu2024stabilitysharperriskbounds}. 
Another line of work concentrates on alternative generalization metrics, such as the gradient norm, rather than the conventional generalization gap. However, they may face challenges when applied to multi-pass training regimes~\citep{lei2023stability,zhang2024generalization}.
In contrast, this paper introduces a novel approach parallel to the aforementioned two lines, based on the notion of \emph{homogeneity}. 
Table~\ref{tab:comparison} summarizes the comparison.

\textbf{Homogeneous Neural Networks.}
The concept of homogeneous neural networks has been explored in the literature~\citep{DBLP:conf/nips/WeiLLM19,DBLP:conf/icassp/RangamaniB22,DBLP:conf/nips/VardiSS22a}.
A line of work focuses on the implicit bias towards max-margin solutions in homogeneous neural networks ~\citep{DBLP:conf/icml/NacsonGLSS19,DBLP:conf/iclr/LyuL20,DBLP:conf/nips/JiT20} and its variants~\citep{DBLP:conf/iclr/KuninYMG23}. 
Of particular relevance here is \citet{DBLP:journals/nn/PaquinCG23}, which investigates algorithmic stability for SGD within the context of homogeneous neural networks. 
Different from their analysis, this paper (a) aims to improve stepsize requirements in stability analysis, and (b) introduces the concept of approximate smoothness to extend the scope of previous regimes.

\section{Homogeneity}
\textbf{Distributions, Datasets and Algorithms.}
Let $(\X, \Y) \sim \cP  \in \bR^{d} \times \bR$ denote the feature-response pair, where $\cP$ denotes the joint distribution. 
We consider a noisy classification problem where $\Y \in \{-1, +1\}$ contains label noise. 
The model $f_\vw(\cdot)$ is trained on an $n$-sample dataset $S = \{ (\X_i, \Y_i)\}_{i \in [n]}$ where the samples are drawn i.i.d. from the distribution $\cP$. 
Here \( [n] \) denotes the set $\{1, 2, \cdots, n\}$. 
We denote $\cD$ as the distribution of dataset $S$. 

We consider a non-negative loss function $\ell(\vw; \X, \Y)$ with the population loss $\cL(\vw) \triangleq \bE_{(\X, \Y) \sim \cP} \allowbreak \ell(\vw; \X, \Y)$ and the empirical loss $\hat{\cL}(\vw) \triangleq \frac{1}{n} \sum_{i \in [n]} \ell(\vw; \X_i, \Y_i)$. 
We assume that the Bayesian optimal on the unit sphere is non-zero, namely $\underline{\sigma}^2 \triangleq \min_{\vw} \cL(\vw/\|\vw\|) \neq 0$. 
This condition generally holds in scenarios with label noise. 

The model is optimized via the SGD algorithm $\cA$. 
Let $\vw_t$ denote the trained weight at iteration $t$, and let $\vv_t \triangleq \vw_t / \| \vw_t\| $ be its corresponding direction. 
Starting from an initialization $\bE_\cA \|\vw_0\|^2 = \cO(1)$, the weight update follows:
\begin{equation}
\label{eqn: iteration}
\vw_{t+1} = \vw_{t} - \eta_t \nabla_{\vw_{t}} \ell_t (\vw_t),
\end{equation}
where $\eta_t$ is the stepsize and $\ell_t$ is the loss function at iteration $t$. 
In this paper, we consider SGD with replacement, and therefore, $\ell_t(\cdot) \triangleq \ell(\cdot; \X_i, \Y_i)$ with probability $1/n$ for each index $i \in [n]$. 
Let $\bE_I$ denote the expectation over the choice of the index $i$ at the current iteration, given the parameters from all previous iterations. 
Let $\bE_\cD$ denote the expectation over the training dataset $S \sim \cD$, and $\bE_\cA$ denote the expectation over the randomness of the algorithm $\cA$, including the initialization and the sampling of indices across all previous iterations. 
Standard asymptotic notations are defined in Appendix~\ref{app:notations}.

\textbf{Homogeneity.}
Our results are based on the homogeneity of neural networks. Next, we provide a detailed definition of homogeneous neural networks and explain how they can be constructed.
\begin{definition}
   \label{def: homogeneous} \textnormal{\textbf{(Homogeneous function and homogeneous neural network)}}
A function $h$ is $\homo$-homogeneous ($H\geq 0$) if $h(c \vw) = c^\homo h(\vw)$.
Specifically, a neural network $\Phi(\vw; \X)$ is $\homo$-homogeneous ($H\geq 0$) if 
    \begin{equation*}
        \Phi(c \vw; \X) = c^\homo \Phi(\vw; \X).
    \end{equation*}
\end{definition}
It's worth noting that the concept of homogeneity is evident across a range of neural network architectures. 
Specifically, neural networks employing ReLU, LeakyReLU, linear layers, and max-pooling layers inherently exhibit homogeneity. 
One could construct an $\homo$-homogeneous neural networks based on Proposition~\ref{prop: homo-property}, and we refer to Figure~\ref{fig: illustration of construction} in Appendix~\ref{app:illustration_sec} for an illustration.

\begin{proposition}[Homogeneity Construction]
\label{prop: homo-property}
Consider two types of layers: 
\begin{enumerate}[leftmargin=12pt]
    \item \textbf{Normalized layers\footnote{We only present linear connections here for simplicity but it allows for different structures, \eg, ResNet connections $\cN_\vw(a) \triangleq a + \frac{\vw}{\| \vw\|}  \sigma_n(a)$ and max-pooling layers.}} $\cN_\vw(a) = \frac{\vw}{\| \vw\|}  \sigma_n(a)$, with arbitrary activation $\sigma_n$, \eg, ReLU, sigmoid; 
    \item \textbf{Unnormalized layers} {$\cU_\vw(a) \triangleq \vw \sigma_u(a)$, where $\sigma_u$ is ReLU, LeakyReLU, Linear activation, or $\sigma(\cdot)=\max\{0,\cdot\}^u$ with $u \ge 1$. }
\end{enumerate}
For convenience, we omit the bias when the context is clear. 
The $h$-layer neural network $\Phi_\homo(\vw; \X)$ with $h-\homo$ normalized layers $\cN_\vw$ and $\homo$ unnormalized layers $\cU_\vw$ is $\homo$-homogeneous, where

\begin{equation*}
\begin{split}
    \Phi_H(\vw; \X) &= \underbrace{\cU_{\vw_{h}} \circ \cU_{\vw_{h-1}} \circ \cdots \circ \cU_{\vw_{h-H+1}}}_{H \text{ unnormalized layers}} \\
    &\quad \circ \underbrace{\cN_{\vw_{h-H}} \circ \cdots \circ \cN_{\vw_{1}}(\X)}_{h-H \text{ normalized layers}}.
\end{split}
\end{equation*}
\end{proposition}

\textbf{Loss Function.}
Following the $\homo$-homogeneous neural network $\Phi(\vw; \X)$, the main result in Theorem~\ref{thm: homo leads to generalization} focuses on the loss function
\begin{equation}
    \label{eq:max_loss}
    \ell(\vw; \X, \Y) = \max\{-\Y \Phi(\vw; \X), 0\}. 
\end{equation}
Extensions to other loss functions are further discussed in the Remark of Theorem~\ref{thm: homo leads to generalization}. 

\textbf{The Validity of the Loss Function.}
The loss function in~\eqref{eq:max_loss} is valid because 
(a) the training direction depends only on the weight direction in classification problems with homogeneous neural networks, and
(b) the generalization metric is defined on the unit sphere, meaning that the bound of the generalization gap does not necessarily converge to zero as $\| \vw \|$ approaches zero.
We choose this loss function because it harnesses the homogeneity of neural networks, namely, applying $\ell$ on an $\homo$-homogeneous neural network returns an $\homo$-homogeneous loss.
Our theorems can be generalized to other loss functions that similarly harness homogeneity, \eg, $(\max\{-\Y \Phi(\vw; \X), 0\})^u$ with $u \geq 1$.
To proceed, we derive Lemma~\ref{lem: L-homogeneous property} based on the homogeneity:

\begin{lemma}
\label{lem: L-homogeneous property}
For any iteration $t > 0$ and a non-negative $\homo$-homogeneous function $\ell_t$ where $\homo \neq 0$, 
define the effective stepsize  $\tilde{\eta}_t$ as the stepsize projected onto the sphere, namely, on $\vv_t = \vw_t / \|\vw_t \|$,
then the following properties hold:
\begin{enumerate}[leftmargin=*]
    \item Effective Stepsize on the sphere: $\tilde{\eta}_t = \eta_t \| \vw_t\|^{\homo-2}$;
    \item Inner Product: $\vw_t^\top \nabla_{\vw_t} \ell_t(\vw_t) = \homo \ell_t(\vw_t)$;
    \item {Norm Iteration over Projected Dynamics}: $\| \vw_{t+1} \|^2 = \|\vw_{t} \|^2 ( 1 + \tilde{\eta}_t^2 \|\nabla_{\vv_t} \ell_t(\vv_t) \|^2  - 2 \homo \tilde{\eta}_t \ell_t(\vv_t))$.
\end{enumerate}
\end{lemma}
Notably, the effective stepsize is normalized by the weight norm.
This normalization enables the possibility of having a large stepsize while maintaining a small effective stepsize. 
This observation constitutes a key aspect of our analysis. 
We defer the proof of Lemma~\ref{lem: L-homogeneous property} to Appendix~\ref{appendix: proof of H-homo}. 

In the context of homogeneity, the weight direction matters much rather than its precise values. This phenomenon aligns with common practices in classification, where we mainly focus on the sign of the prediction instead of its precise value.
Consequently, we evaluate the population loss and empirical loss on the sphere, as outlined in \eqref{eqn:lossevaluation}. 
\begin{equation}
    \label{eqn:lossevaluation}
    \begin{split}
        \cL^s(\vw) &= \cL(\vw / \|\vw\|) = \cL(\vv), \\
        \hat{\cL}^s(\vw) &= \hat{\cL}(\vw / \|\vw\|) = \hat{\cL}(\vv).
    \end{split}
\end{equation}
\textbf{Algorithmic Stability.}
Algorithmic stability is a popular technique for bounding generalization error by measuring an algorithm's sensitivity to perturbations in the training data~\citep{DBLP:journals/jmlr/BousquetE02,shalev2010learnability,DBLP:conf/icml/HardtRS16}. This paper relies on this framework to motivate our analysis. We summarize the formal definitions of uniform stability ($\epsilon_T$-stable) and the classic stability bounds for SGD in Appendix~\ref{app:stability-review}. 

\textbf{Why Algorithmic Stability Requires an $\cO(1/t)$ Stepsize under Non-Convexity?}
The distinction between convex and non-convex optimization under algorithmic stability is fundamentally linked to the expansive property, as detailed in Lemma 3.6 in \citet{DBLP:conf/icml/HardtRS16}.
In the context of convex training, if the same training sample is selected (with a probability of $1 - \frac{1}{n}$), the following inequality holds: $\|\vw_{t+1} - \vw_{t+1}^\prime\| \leq \|\vw_t - \vw_t^\prime\|$.
However, in the case of non-convex training, the inequality is modified to: $\|\vw_{t+1} - \vw_{t+1}^\prime\| \leq (1 + \beta \eta_t) \|\vw_t - \vw_t^\prime\|$.
This modification implies that errors can accumulate over iterations in non-convex training, leading to an accumulation term $\prod_{t \in [T]} (1 + \beta \eta_t) \approx \exp(\beta \sum_{t\in[T]} \eta_t)$.
To ensure that this term grows polynomially with the number of iterations $T$, we require that the sum of the stepsizes over the iterations is $\sum_{t\in[T]} \eta_t = \cO(\log T)$, which in turn necessitates a stepsize of $\eta_t = \cO(1/t)$.

\section{Generalization Under Homogeneity}
\label{sec: theory}
In this section, we present our main theorems, which establish the connection between homogeneity and generalization. 
We begin by introducing the necessary assumptions in Section~\ref{sec: assumption}. 
Next, we present our main generalization bound for general $H$-homogeneous neural networks in Section~\ref{sec: Homogeneity Leads to Generalization}. 
We then specialize this result to the important case of three-homogeneous neural networks in Section~\ref{sec: three homo leads to gen}. 
Finally, we compare the optimization performance in Section~\ref{sec: optimization} to demonstrate the separation between different stepsize choices.  

\subsection{Assumptions}
\label{sec: assumption}
Before delving into the main theorem, we introduce the following assumptions concerning the loss function and the optimization process. 

\begin{assumption}[Loss Bound]
\label{assump: loss bound}
Assume that the individual loss is upper bounded by a constant $\bar{\sigma}^2$ on the unit sphere for any sample, namely
\begin{equation*}
    \sup_{\vv: \|\vv \| = 1} \sup_{\X,\Y} \ell(\vv; \X, \Y) \leq \bar{\sigma}^2.
\end{equation*}
Additionally, assume that the training loss is lower bounded by half of the Bayesian optimal during the training process, namely, for $t \in [0, T]$ with a given iteration $T$,
\begin{equation*}
    \inf_{t} \hat{\cL}(\vv_{t}) \geq \frac{1}{2} \underline{\sigma}^2. 
\end{equation*}
\end{assumption}
Assumption~\ref{assump: loss bound} provides bounds for the training loss. 
The upper bound is standard in related analyses, ensuring that the training procedure does not result in excessively high loss values that could lead to an unstable optimization process. 
The lower bound is necessary to obtain a consistent generalization bound in noisy label settings, since without which the generalization gap would not converge to zero, as demonstrated in Proposition~\ref{prop: loss bound requirement}. 
Notably, there is a small gap here where we assume the lower bound along the trajectory, while Proposition~\ref{prop: loss bound requirement} only accounts for the necessity of the lower bound on the iteration $T$.
We argue that this gap is mild since the training loss usually decreases. 

\begin{proposition}
\label{prop: loss bound requirement}
Given the Bayesian optimal $\underline{\sigma}^2$, for any given iteration $t$, if the training loss satisfies $\hat{\cL}(\vv_{t}) \leq \frac{1}{2} \underline{\sigma}^2$, the corresponding generalization gap satisfies $|\cL(\vv_t) - \hat{\cL}(\vv_t) | \geq \frac{1}{2} \underline{\sigma}^2$.
\end{proposition}

\begin{assumption}\textnormal{\textbf{(Lipschitz, Approximately Smooth)}}
\label{assump: approximate smooth and lipshitz}
Assume that the loss function $\ell(\cdot; \X, \Y)$ is $L$-Lipschitz on the sphere for any sample $(\X, \Y)$, namely, 
\begin{equation*} 
\sup_{\vv: \| \vv\|=1}\| \nabla_{\vv} \ell(\vv; \X, \Y) \| \leq L.
\end{equation*}
Additionally, we assume that the loss function is $(\gamma,\beta)$-approximately smooth on the sphere for any sample $(\X, \Y)$, meaning that for each $\ell(\vv; \X, \Y)$, there exists a globally $\beta$-smooth function $\bar{\ell}(\vv; \X, \Y)$ such that for each $\vv$
\begin{equation*}
    \sup_{\vv: \| \vv\|=1}\| \nabla_\vv \ell(\vv; \X, \Y) - \nabla_\vv \bar{\ell}(\vv; \X, \Y) \| \leq \gamma . 
\end{equation*} 
\end{assumption}
The Assumption~\ref{assump: approximate smooth and lipshitz} introduces constraints on the loss function to ensure its Lipschitz and approximate smooth properties, which are widely used in related literature~\citep{DBLP:conf/icml/HardtRS16}. 
Note that Assumption~\ref{assump: approximate smooth and lipshitz} degenerates into the standard smoothness assumption when $\gamma = 0$.
We extend the previous smoothness assumption to approximate smoothness to accommodate the non-smooth behavior of ReLU activations.
It is worth noting that there exist smooth activation functions that maintain homogeneity, \eg, $\sigma(\cdot) = (\max\{\cdot, 0 \})^u$ with $u > 1$, which opens up possibilities where $\gamma = 0$. 

\subsection{Generalization Under Homogeneity}
\label{sec: Homogeneity Leads to Generalization}
This section provides the generalization bounds for general homogeneous neural networks. 
We defer the complete proof to Appendix~\ref{sec: proof sketch}. 
\begin{theorem}[Generalization Under Homogeneity]
\label{thm: homo leads to generalization}
Assume that the neural network $\Phi(\vw; \X)$ is $\homo$-homogeneous for any $\X$, with $\homo > 2$.
Given any fixed $T>0$, let Assumption~\ref{assump: loss bound} hold with loss bounds $\underline{\sigma}^2$ and $\bar{\sigma}^2$, and Assumption~\ref{assump: approximate smooth and lipshitz} hold with $L$-Lipschitz and $(\gamma,\beta)$-approximately smooth loss, where $\gamma = o(T/n)$.
Then, there exist infinitely many stepsizes satisfying $\bE_{\cA, \cD} \eta_t = \Omega(t^{\frac{\homo-4}{2}})$, for which the generalization bound holds:
\begin{equation*}
\bE_{\cA,\cD} \cL^s(\vw_T) - \hat{\cL}^s(\vw_T) \leq \left[L m_2\right]^{\frac{1}{m_1 + 1}}\left[1 + \frac{1}{m_1}\right] \frac{T^\frac{m_1}{m_1 + 1}}{n},
\end{equation*}
where $m_1 = \frac{4c_2}{\underline{\sigma}^2} \left[\bar{\sigma}^2 + \frac{\beta}{\homo}\right]$ and $m_2 = \frac{8c_2}{\homo \underline{\sigma}^2} \left[L + \gamma n\right]$, with constant $c_2 \geq 1$ related to the stepsize.
\end{theorem}

\textbf{Remark: Validity for Smaller Stepsizes.} 
While Theorem~\ref{thm: homo leads to generalization} specifies stepsizes satisfying $\bE \eta_t=\Omega(t^{\frac{\homo-4}{2}})$ to achieve $\Theta(1/t)$ effective stepsizes, the generalization result remains valid for smaller stepsizes, such as $\Theta(t^{\frac{\homo-4}{2}})$.  
This is because a smaller stepsize induces a smaller effective stepsize, which naturally ensures generalization, even if it may lead to slower optimization. 

\textbf{Remark: Extensions of Loss Functions.}
The core derivation can be extended to other loss functions beyond $\ell_t\left(z\right)=\max\{-\Y_t z, 0\}$. 
For a general loss form $\ell_t(\Phi(\vw_t))$ with a homogeneous neural network $\Phi(\vw_t)$, the effective stepsize $\tilde{\eta}_t$ is derived as 
\begin{equation*}
    \tilde{\eta}_t = \rho_t \|\vw_t\|^{H-2}\eta_t,
\end{equation*}
where $\rho_t \triangleq \ell'_t(\Phi(\vw_t)) /{\ell'_t(\Phi(\vv_t))}$. 
Notably, the effective stepsize in Lemma~\ref{lem: L-homogeneous property} is a specific case with $\rho_t=1$. 
For the hinge loss $\ell_t(z) = \max\{0, 1 - \Y_t z\}$, we assume that the zero-loss sets $\{ \vw_t: \ell(\Phi(\vw_t))=0\}$ and $\{\vv_t: \ell(\Phi(\vv_t))=0\}$ approximately overlap. 
This implies the ratio $\rho_t = \ell'_t(\Phi(\vw_t)) / \ell'_t(\Phi(\vv_t))  = 1$ and the property is preserved whenever $\ell_t(\Phi(\vw_t)) \neq 0$ and $\ell_t(\Phi(\vv_t)) \neq 0$. 
We next provide in Corollary~\ref{cor: extension of loss} a relaxed version of Theorem~\ref{thm: homo leads to generalization} regarding the loss form.

\begin{corollary}
\label{cor: extension of loss}
Under the assumptions of Theorem~\ref{thm: homo leads to generalization}, if the loss function $\ell(\cdot)$ satisfies $0 < \rho_t=\ell_t'(\Phi(\vw_t)) / \ell_t'(\Phi(\vv_t)) \allowbreak \le k_1$ with constant $k_1 \ge 1$, and $z \ell'_t(z) \le k_2 \ell_t(z)$ with constant $k_2 > 0$ uniformly for all $t \in [T]$, the generalization bound in Theorem~\ref{thm: homo leads to generalization} remains valid with modified constants: 
\begin{equation*}
    m_1 = \frac{4 c_2k_1}{\underline{\sigma}^2} \left[ k_1 k_2\bar{\sigma}^2 + \frac{\beta}{H} \right], \quad m_2 = \frac{8 c_2k_1}{H \underline{\sigma}^2} (L + \gamma n).
\end{equation*}
Notably, these conditions can be satisfied by many homogeneous loss functions, including $\ell_t(z)=(\max\{-\Y_t z, 0\})^u$ with $u \geq 1$. 
\end{corollary}
\textbf{Comparison with $T/n$-type Bound.}
Under reasonable assumptions, one can construct $T/n$-type generalization bounds via algorithmic stability.
The intuition is that for small iteration $T$, the difference between $S$ and $S^\prime$ is selected with probability $T/n$~\citep{DBLP:conf/icml/HardtRS16}.
However, this type of bound remains valid only when $T = o(n)$, indicating that not all samples are utilized during training. 
Clearly, this does not accurately reflect real-world scenarios.
As for the results in Theorem~\ref{thm: homo leads to generalization}, when $L$, $\homo$, $\beta$, $\bar{\sigma}^2$, and $\underline{\sigma}^2$ are all on a constant scale, the bound is approximately of order $\tilde{\gamma}^{\frac{1}{m_1 + 1}} [T/n]^\frac{m_1}{m_1 + 1}$, where $\tilde{\gamma} = \max\{\gamma, 1/n\}$.
Notably, this new bound remains consistent when setting $T = o(n/\tilde{\gamma}^{1/m_1})$, allowing for multiple uses of each sample during the training process. 

\subsection{Generalization Under Three-Homogeneity}
\label{sec: three homo leads to gen}
This section specializes the general bound in Theorem~\ref{thm: homo leads to generalization} to three-homogeneous neural networks (Corollary~\ref{cor: Three-Homogeneous Leads to Generalization}). 
We focus on this specific case for two main reasons: (a) setting $H=3$ could yield the stepsize $\bE_{\cA, \cD} \eta_t = \Omega(1/\sqrt{t})$, a rate that commonly appears in theoretical analysis \citep{DBLP:journals/siamjo/NemirovskiJLS09}; 
and (b) three-homogeneous neural networks themselves are  widely adopted in practice, such as VGG19~\citep{DBLP:journals/corr/SimonyanZ14a}.
This alignment is somewhat surprising since the above two reasons are independent. 
We leave further discussion on this alignment for future work.
\begin{figure*}
\vspace{1.5em}
    \subfloat[Normalization]{\includegraphics[width=0.5\textwidth]{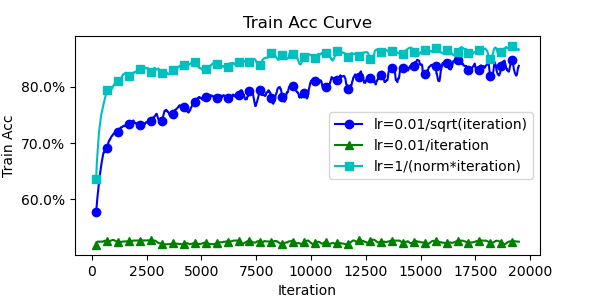}}
    \subfloat[Loss curve illustration]{\includegraphics[width=0.45\textwidth]{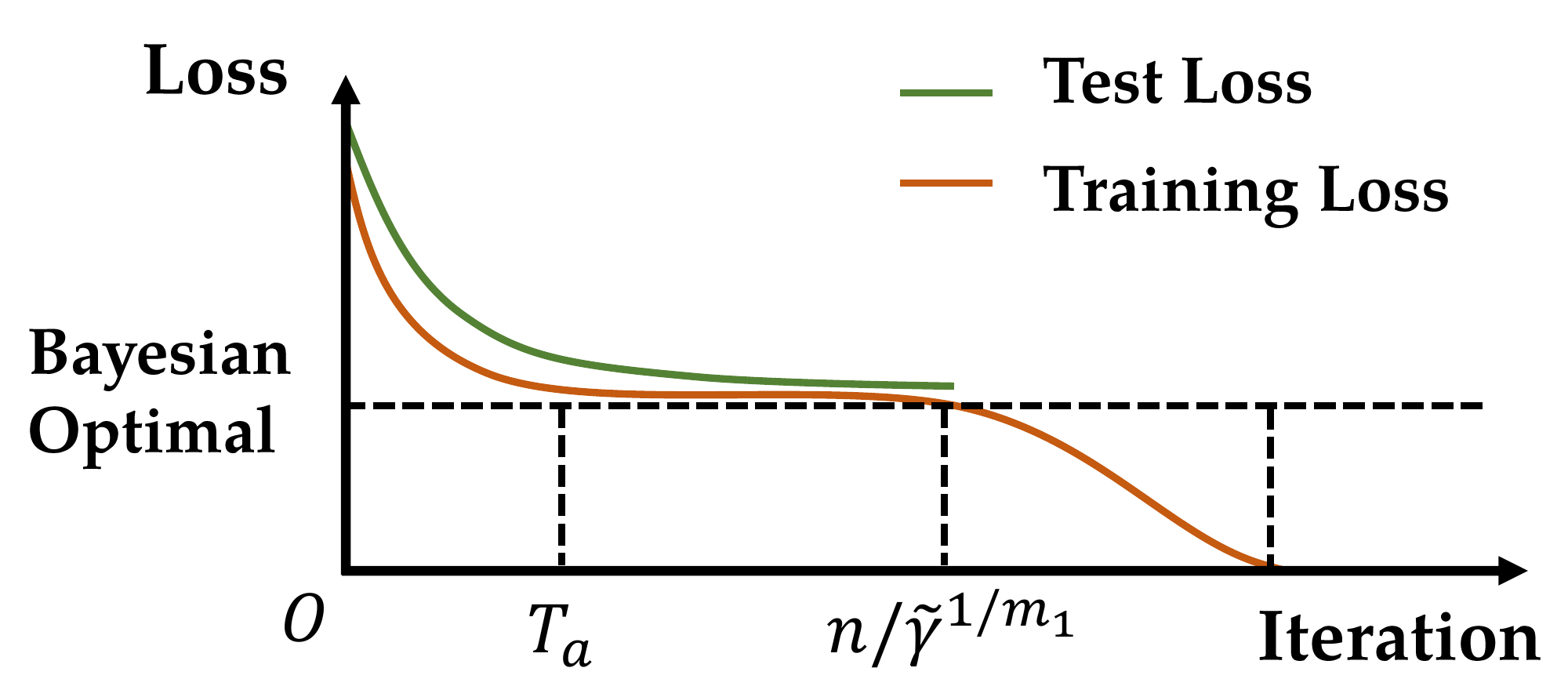}}
    \caption{(left) Training accuracy of SGD with different schedulers over a three-layer ReLU network. The initial learning rate is selected via grid search over $\{10^{-5}, 10^{-4},\cdots, 10^1\}$. SGD with stepsize $\Theta(1/t\Vert  \vw_t \Vert )$ achieves similar training accuracy as SGD with stepsize $\Theta(1/\sqrt{t})$, both outperforming SGD with stepsize $\Theta(1/t)$; (right) When $T_a = o(n/\tilde{\gamma}^{1/m_1})$, the training and test loss curve does not decrease at $t \in (T_a, o(n/\tilde{\gamma}^{1/m_1}))$.}
    \label{fig: loss curve}
    \vspace{0.5em}
\end{figure*}
\begin{corollary}
\label{cor: Three-Homogeneous Leads to Generalization}
Consider a three-homogeneous neural network as constructed in Proposition~\ref{prop: homo-property} with $\homo=3$. 
Let $T_a$ denote the iteration at which the training error achieves a non-zero Bayesian optimal value\footnote{Here we allow for an $o(1)$ relaxation.}, namely,
\begin{equation*}
    T_a = \argmin_t \left\{ \bE_{\cA, \cD} \hat{\cL}(\vv_t) \leq \underline{\sigma}^2 \right\},
\end{equation*}
where we omit the dependency on stepsize and dataset for simplicity. 
Let Assumption~\ref{assump: loss bound} and Assumption~\ref{assump: approximate smooth and lipshitz} hold with $(\gamma,\beta)$-approximately smooth loss, where $\gamma = o(T_a / n)$.
Then, there exist infinitely many stepsizes satisfying $\bE_{\cA, \cD} \eta_t = \Omega(t^{-1/2})$, such that if $T_a = o(n/\tilde{\gamma}^{1/m_1})$, where $\tilde{\gamma} = \max\{\gamma, 1/n\}$  and $m_1$ denote a constant related to loss bound and smoothness, 
it holds that
\begin{equation*}
\lim_{n \to \infty} \bE_{\cA, \cD} \cL^s(\vv_{T_a}) = \underline{\sigma}^2.
\end{equation*}
\end{corollary}

Corollary~\ref{cor: Three-Homogeneous Leads to Generalization} demonstrates that the generalization gap would be small within $o(n/\tilde{\gamma}^{1/m_1})$ iterations.
Therefore, if the training loss achieves the Bayesian optimal within $o(n/\tilde{\gamma}^{1/m_1})$ iterations, the corresponding test loss would also achieve the Bayesian optimal value. 
Additionally, $o(n/\tilde{\gamma}^{1/m_1})$ enables multiple uses of each sample during the training given $\gamma = o(1)$, which is rarely achieved in previous works with a similar order of stepsize~\citep{lei2021learning,lei2023stability}. 

\textbf{The Role of Normalization.}
Doubts might arise regarding whether the results in Corollary~\ref{cor: Three-Homogeneous Leads to Generalization} suggest that a three-layer MLP suffices in real-world scenarios, potentially rendering deeper architectures obsolete.
To address this, we emphasize the pivotal role that normalization layers play. 
The significance lies in two key aspects. 
Firstly, they introduce non-linearity, effectively reducing approximation errors. 
Secondly, increasing the depth of a network might facilitate the optimization process~\citep{DBLP:conf/icml/AroraCH18}, leading to faster attainment of Bayesian optimality in training error (smaller $T_a$).

\textbf{Loss Curve.}
One can observe from Theorem~\ref{thm: homo leads to generalization} that the generalization gap remains within $o(1)$ for cases where $T = o(n/\tilde{\gamma}^{1/m_1})$. This observation implies that if $T_a$ is relatively small, indicating a fast convergence of the training loss to its Bayesian optimal, it cannot sustain this decrease until the point where $T$ reaches $\Theta(n/\tilde{\gamma}^{1/m_1})$. Otherwise, such a trend would result in a test error smaller than the Bayesian optimal, which is impossible. 
{We refer to Figure~\ref{fig: loss curve} (right) for an illustration.}
This phenomenon is empirically observed in \citet{DBLP:conf/iclr/WenTZ23}.

\subsection{Additional Discussions on Optimization Performance}
\label{sec: optimization}
This section demonstrates a separation in optimization performance between the stepsize in this paper and those in prior works. 
Following the setting $\homo=3$ in Section~\ref{sec: three homo leads to gen}, we prove that a stepsize of order $\bE_{\cA, \cD} \eta_t=\Omega(1/\sqrt{t})$ potentially optimizes better compared to $\bE_{\cA, \cD} \eta_t=\cO(1/t)$. 
For convenience of discussion, we consider the class of stepsizes that could eliminate the randomness of the corresponding effective stepsize. 
The result is provided in Theorem~\ref{thm: separation}. 

\begin{theorem}\textnormal{\textbf{(Separation Between Stepsize $\Omega(1/\sqrt{t})$ and $\cO(1/t)$)}}
\label{thm: separation}
Under the settings of Theorem~\ref{thm: homo leads to generalization} with $H=3$, assume that the individual loss function $\ell_i(\cdot)=\ell(\cdot; \X_i, \Y_i)$ on the sphere ($\| \vv\| = 1$) satisfies

1. Smoothness: $\| \nabla_{\vv_1} \ell_i(\vv_1) - \nabla_{\vv_2} \ell_i(\vv_2) \| \leq \beta \| \vv_1 - \vv_2 \|$; 

2. PL condition: $\frac{1}{2}\| \nabla_{\vv} \ell_i(\vv)\|^2 \geq \mu (\ell_i(\vv) - \min_\vv \ell_i(\vv))$; 

3. Strong growth: $\max_{i \in [n]} \{ \|  \nabla_{\vv} \ell_i(\vv)\| \} \leq B \| \nabla_\vv \hat{\cL}(\vv) \|$;

4. $\bE_I \hat{\cL}(\vv_t)  \ge (1+\alpha)/\alpha \hat{\cL}(\vv^*)$,

then for SGD with $\bE_{\cA} \| \vw_0 \|^2 = \cO(1)$ and $\tilde{\eta}_t \leq \min\{\frac{1}{\beta B^2}, \allowbreak \frac{\homo \underline{\sigma}^2}{2 L^2}\}$, if $\mu > 4B^2 \homo^2\bar{\sigma}^2(1+\alpha)$, it holds that for some given constant $c>0$, 
\begin{itemize}[leftmargin=*]
    \item For learning rates satisfying $\bE_{\cA, \cD} \eta_t = \Omega(1/\sqrt{t})$ used in Theorem~\ref{thm: homo leads to generalization}, the convergence rate would be 
\begin{equation*}
    \bE_{\cA, \cD} \hat{\cL}(\vv_{T}) - \min_\vv \hat{\cL}(\vv) \leq \frac{\bE_{\cA, \cD} (\hat{\cL} (\vv_0)-\min_\vv \hat{\cL}(\vv) )}{T^{c\lambda}},
\end{equation*}
\item There exist infinitely many learning rates satisfying $\bE_{\cA, \cD} \eta_t = \cO(1/t)$, leading to the convergence rate 
\begin{equation*}
    \bE_{\cA, \cD} \hat{\cL}(\vv_{T}) - \min_\vv \hat{\cL}(\vv) \leq \frac{\bE_{\cA, \cD} (\hat{\cL} (\vv_0)- \min_\vv \hat{\cL}(\vv)  )}{[\log T]^{c\lambda}}.
\end{equation*}
\end{itemize}
where $\lambda = \frac{\mu}{B^2} -4\homo^2\bar{\sigma}^2(1+\alpha)$.
\end{theorem}

Theorem~\ref{thm: separation} establishes a separation between the effects of different stepsizes during the optimization process. 
It asserts that there are certain training landscapes for which an optimization algorithm with stepsize $\Omega(1/\sqrt{t})$ achieves a significantly better convergence rate compared to one with stepsize $\cO(1/t)$.
The PL condition is used to ensure convergence, while the strong growth condition is utilized to enhance the convergence rate of SGD. 
Notably, this result pertains to optimization performance, distinguishing it from prior works that apply the PL condition to derive generalization results via algorithmic stability. 
We provide the proof in Appendix~\ref{app:seperation}. 

\begin{example}[Compatibility between Optimization and Generalization]
\label{ex: compatibility_2d}
Consider the 3-homogeneous loss function:
\begin{equation*}
    \ell(\vv) = (v_1^2 + 2v_2^2)^{3/2}.
\end{equation*}
This case illustrates that the optimization guarantees in Theorem~\ref{thm: separation} and the generalization requirements in Corollary~\ref{cor: Three-Homogeneous Leads to Generalization} can be satisfied simultaneously. 
In the unit sphere, this loss satisfies the required regularity conditions with $\underline{\sigma}^2=1\leq \ell(\vv)\le2^{3/2}=\bar{\sigma}^2$, $L=6\sqrt{2}$, $\gamma=0$, $\beta = 12\sqrt{2}$ and $\mu \approx 19.7$. 
Without affecting asymptotic behavior, relax the definition of hitting time in Corollary~\ref{cor: Three-Homogeneous Leads to Generalization} to 
\begin{equation*}
    T_a(\tau(n)) = \min \{ t : \bE \hat{\cL}(\vv_t) \leq \underline{\sigma}^2 + \tau(n) \}
\end{equation*}
for some $\tau(n) \rightarrow 0$. 
Theorem~\ref{thm: separation} implies $\ell(\vv_T)-\underline{\sigma}^2=\cO(T^{-c\mu})$, which yields $T_a \leq \tau(n)^{-1/c\mu}$ for a constant $c$. 
To satisfy the generalization requirement $T_a = o(n/\tilde{\gamma}^{1/m_1})$ with $\tilde{\gamma} = 1/n$, we can choose 
\begin{equation*}
  \tau(n) = n^{-c\mu(1 + \epsilon/m_1)}
\end{equation*}
for some $\epsilon \in (0, 1)$.   
This ensures $T_a \leq n^{1 + \epsilon/m_1} = o(n^{1 + 1/m_1})$.  
This argument naturally extends to $\ell(\vv)=(v_1^2+2\Sigma_{i=2}^Dx_i^2)^{3/2}$. 
\end{example}
\section{Generalization Beyond Lipschitz}
\label{sec:relax-lipschitz}
This section demonstrates that the stepsize regime in Theorem~\ref{thm: homo leads to generalization} ensures generalization even without the Lipschitz assumption. 
The analysis proceeds by first defining a stability metric and establishing its relationship with generalization in Section~\ref{sec:gen_via_conditional_stability}, subsequently deriving the non-Lipschitz generalization bound in Section~\ref{sec:gen_under_non_lipschitz}. 

\subsection{Stability and Generalization}
\label{sec:gen_via_conditional_stability}
This section starts from the definition of the conditional on-average stability, adapted from the \emph{on-average model stability} \citep{DBLP:conf/icml/LeiY20} in Appendix~\ref{app:on-avg-background}. 

\begin{definition}[Conditional On-Average Stability]
\label{def: stability}
Let $S = \{z_1, \dots, z_n\}$ and $S^{(i)} = \{z_1, \dots,\allowbreak z'_i, \dots, \allowbreak z_n\}$ be two datasets differing only at the $i$-th index, where $z_i \triangleq (\X_i, \Y_i)$. 
$\vv_T$ and $\vv_T^{(i)}$ are trained on $S$ and $S^{(i)}$ over $T$ iterations, respectively. 
For any given $t_0 \in [T]$, let $\mathcal{E}_{(i)}$ be the event that the index $i$ is not selected during the first $t_0$ iterations. 
The algorithm $\cA$ is said to be $\epsilon_{\text{stab}}$-conditionally stable on average if:
\begin{equation*}
\epsilon^2_{\text{avg}, T \mid t_0} \triangleq \frac{1}{n} \sum_{i=1}^n \bE_{\cA, \cD} \left[ \|\vv_T - \vv_T^{(i)}\|^2 \mid \mathcal{E}_{(i)} \right] \leq \epsilon_{\text{stab}}.
\end{equation*}
\end{definition}
The following Lemma~\ref{thm:gen_decomposition} connects the generalization gap to the conditional on-average stability. 
We focus on the setting where $\gamma=0$. 
Notably, the assumptions of homogeneity and $\beta$-smoothness can hold simultaneously, \eg, with the loss $\ell_t(z)=(\max\{-\Y_t z, 0\})^2$ and activation $\sigma(\cdot)=\max\{0,\cdot\}^2$. 

\begin{lemma}
\label{thm:gen_decomposition}
Under Bounded Loss (Assumption~\ref{assump: loss bound}) and $\beta$-smoothness, for any $t_0 > 0$ and $\zeta > 0$, the generalization bound holds:
\begin{equation*}
\bE_{\cA,\cD} [\cL^s(\vw_T) - \hat{\cL}^s(\vw_T)] 
    \le 
    \left(\frac{t_0}{n} + \frac{\beta}{\zeta}\right)\bar{\sigma}^2 
    + 
    \frac{\beta+\zeta}{2} \epsilon^2_{\text{avg}, T \mid t_0},
\end{equation*}
\end{lemma}
where $\zeta$ is a free parameter. The proof is deferred to Appendix~\ref{subsec:proof_thm_5_1}. 

\subsection{Generalization Beyond Lipschitz}
\label{sec:gen_under_non_lipschitz}
This section derives the generalization bounds without the Lipschitz assumption. 
\begin{theorem}[Generalization Beyond Lipschitz]
\label{thm:gen_without_lipschitz}
Given any fixed $T > 0$, let Bounded Loss (Assumption~\ref{assump: loss bound}) and $\beta$-smoothness hold. Then, for an $H$-homogeneous network ($H>2$) and the stepsize regime in Theorem~\ref{thm: homo leads to generalization}, the generalization gap satisfies:
\begin{equation*}
\bE_{\cA,\cD} [\cL^s(\vw_T) - \hat{\cL}^s(\vw_T)] = \cO \left( n^{-\frac{m'_1+2}{m'_1+3}} \mathcal{K}_T^{\frac{1}{m'_1+3}} T^{\frac{m'_1}{m'_1+3}} \right),
\end{equation*}
where $\mathcal{K}_T = 1 + T/n$ and $m'_1 = 8\frac{c_2\bar{\sigma}^2}{\underline{\sigma}^2} + \frac{4c_2\beta}{H\underline{\sigma}^2}$. 
\end{theorem}
Notably, this bound still supports multi-pass training.
We defer the proof to Appendix~\ref{subsec:proof_thm_5_2}. 

\textbf{Remark: Comparison of Bounds.}
Theorem~\ref{thm:gen_without_lipschitz} removes the Lipschitz requirement in Assumption~\ref{assump: approximate smooth and lipshitz}, better reflecting practical scenarios. 
Interestingly, despite this relaxation, it sometimes yields a faster growth in the iteration upper bound compared to that in Theorem~\ref{thm: homo leads to generalization}. 
Specifically, it supports $T = o(n^{1 +2/(m'_1+1)})$ while ensuring generalization, which surpasses $o(n^{1+1/m_1})$ when $4c_2\beta > \underline{\sigma}^2 \homo$, allowing for more training epochs. 
This improvement potentially stems from the different stability metrics utilized. 
It reveals that, in our setting, uniform stability may contain redundancy relative to the Assumptions~\ref{assump: loss bound} and~\ref{assump: approximate smooth and lipshitz}. 
We leave the application of our stepsize regime to other stability metrics for further exploration. 
\section{Conclusion}
\label{sec: diss}
In this paper, we establish generalization bounds for homogeneous neural networks. 
We prove in Theorem~\ref{thm: homo leads to generalization} that for $\homo$-homogeneous neural networks, there exist stepsizes of order $\Omega(t^{\frac{H-4}{2}})$ that ensure generalization via algorithmic stability. 
This is grounded in the observation that while the stepsize may be large, its effective stepsize on the unit sphere might be considerably small. 
Applying the main result to the specific case of $\homo=3$, Corollary~\ref{cor: Three-Homogeneous Leads to Generalization} establishes generalization for stepsizes of order $\Omega(1/\sqrt{t})$, relaxing the previous requirement of $\cO(1/t)$. 
Our results also extend to a broader range of loss functions, activation functions, and architectures, as detailed in Corollary~\ref{cor: extension of loss} and Proposition~\ref{prop: homo-property}. 
Besides, we provide in Theorem~\ref{thm: separation} an optimization result on the separation between stepsize $\cO(1/t)$ and stepsize $\Omega(1/\sqrt{t})$. 
Furthermore, we strengthen our framework in Section~\ref{sec:relax-lipschitz} by deriving non-Lipschitz generalization bounds, better aligning our theory with real-world scenarios. 
Overall, our findings highlight homogeneity as a helpful property that enhances algorithmic stability, yielding benefits in both generalization and optimization. 

\bibliography{reference}
\bibliographystyle{unsrtnat}
\clearpage

\newpage
\appendix
\renewcommand{\appendixpagename}{\centering Appendix}
\appendixpage
\startcontents[section]
\printcontents[section]{l}{1}{\setcounter{tocdepth}{2}}

\newpage
\begin{center}
    \begin{huge}
        Appendix
    \end{huge}
\end{center}
\section{Additional Discussions}
This appendix provides additional discussion in Section~\ref{app:extend_rw} on algorithmic stability and generalization analysis to complement Section~\ref{sec:related_works}. 
Subsequently, Section~\ref{app:illustration_sec} offers a detailed explanation and visualization of the homogeneous neural network construction introduced in Proposition~\ref{prop: homo-property}.

\subsection{Additional Related Work}
\label{app:extend_rw}
\textbf{Algorithmic Stability.} Algorithmic stability is a prominent generalization technique that may generalize to a broader range of loss functions and models in multi-pass settings~\citep{DBLP:journals/jmlr/BousquetE02,DBLP:conf/colt/FeldmanV19,DBLP:conf/colt/BousquetKZ20,bassily2020stability,DBLP:conf/iclr/TengMY22,yangstability}. 
This argument shows that if the model and the training method are not overly sensitive to data perturbations, the generalization error can be effectively bounded. 
It is provably effective under Lipschitz, convex, and smooth regimes~\citep{DBLP:conf/icml/HardtRS16,DBLP:conf/icml/AsiFKT21,DBLP:journals/corr/abs-2303-10758}. 
There is a line of work focusing on relaxing the constraints of Lipschitz condition~\citep{DBLP:conf/icml/LeiY20,DBLP:conf/nips/AroraBGMU22,DBLP:conf/iclr/NikolakakisHKK23} and smoothness~\citep{DBLP:conf/aistats/YangLLY21,DBLP:conf/nips/WangLYZ22,DBLP:conf/alt/LowyR23}. 

\textbf{Generalization in Stochastic Optimization.} 
Generalization in stochastic optimization has been thoroughly studied~\citep{shalev2010learnability}, across various scenarios such as one-pass SGD \citep{pillaud2018exponential,lugosi2024convergence}, multi-pass SGD \citep{pillaud2018statistical,sekhari2021sgd,lei2021generalization}, DPSGD \citep{bassily2019private,ma2022dimension}, and ERM solutions \citep{feldman2016generalization,aubin2020generalization}. Early research primarily focused on gradient descent in convex learning problems \citep{amir2021never}, where generalization can be bounded by uniform convergence \citep{DBLP:conf/nips/BartlettFT17,DBLP:conf/nips/NagarajanK19,DBLP:conf/iclr/WeiM20,DBLP:conf/iclr/GlasgowWW023}. However, this approach struggles in high-dimensional spaces and can be inadequate \citep{shalev2010learnability,feldman2016generalization}. 
An alternative is online-to-batch conversion \citep{nemirovskij1983problem}, which mitigates the high-dimensional challenge and achieves minimax sample complexity, but it is limited to single-pass training, whereas in practice, longer training often results in better generalization \citep{hoffer2017train}. Recent works aim to close this gap by bounding generalization in multi-pass settings \citep{DBLP:conf/iclr/SoudryHNS18,DBLP:conf/iclr/LiLL21,DBLP:conf/nips/LyuLWA21,sekhari2021sgd}, showing that gradient descent benefits from implicit bias. However, characterizing this implicit bias for more complex models remains an open problem.
In addition to uniform convergence and implicit bias, various other approaches to generalization have been proposed, including PAC-Bayes \citep{DBLP:conf/colt/McAllester99,DBLP:conf/uai/DziugaiteR17,DBLP:journals/entropy/HaddoucheGRS21,DBLP:conf/nips/LotfiFKPGW22}, information-theoretic methods \citep{DBLP:conf/aistats/RussoZ16,DBLP:conf/nips/XuR17,DBLP:conf/nips/NegreaHDK019,DBLP:conf/nips/HaghifamNK0D20,DBLP:conf/isit/HaghifamMRD22,DBLP:conf/alt/HaghifamGTS0D23,lugosi2022generalization}, compression-based techniques \citep{DBLP:conf/icml/Arora0NZ18,DBLP:conf/iclr/0001JTW21}, and benign overfitting \citep{bartlett2020benign,DBLP:conf/colt/ZouWBGK21,koren2022benign,xu2022models,DBLP:conf/iclr/WenTZ23}.
Notably, as one of the most fundamental issues in machine learning theory, generalization theory focuses on the mechanism of deep learning, distinguishing itself from practical approaches such as validation tricks~\citep{DBLP:journals/cacm/ZhangBHRV21,chatterjee2022generalization}.

\subsection{Detailed Construction and Visualization of \texorpdfstring{$H$}{H}-Homogeneous Networks}
\label{app:illustration_sec}
First, we detail the underlying operations for the linear connections constructed in Proposition~\ref{prop: homo-property}. 
For an unnormalized layer, we define $\cU_\vw(a) \triangleq \vw \sigma_u(a)$. 
Here, $a$ denotes the input vector (the output of the previous layer), $\vw$ represents the weight matrix of the \emph{current} layer, and $\sigma_u(\cdot)$ is an element-wise activation function. 
This operation proceeds by activating $a$ followed by matrix multiplication with $\vw$. 
In contrast, the normalized layer $\cN_\vw(a) \triangleq \frac{\vw}{\| \vw\|} \sigma_n(a)$ incorporates an additional normalization step for $\vw$ prior to multiplication. 
This can be implemented either by dividing $\vw$ directly by its Frobenius norm or by normalizing each row independently. 
While both approaches ensure 0-homogeneity, they differ in signal propagation: the latter maintains the input's scale, whereas the former significantly attenuates the output scale relative to the layer's dimensions.

Next, We provide a visualization for the specific case with $H=3$ to illustrate our construction. 
This case is of particular interest as it leads to the $\Omega(1/\sqrt{t})$ stepsize regime discussed in Corollary~\ref{cor: Three-Homogeneous Leads to Generalization}.

As shown in Figure~\ref{fig: illustration of construction}, the architecture is decoupled into two segments:
\begin{itemize}
    \item \textbf{Normalized:} The initial $h-3$ layers utilize normalization $\cN_\vw$ to ensure $0$-homogeneity, allowing for complex architectural features such as skip-connections or diverse activations without increasing the overall degree of the network.
    \item \textbf{Unnormalized:} The final $3$ layers are unnormalized $\cU_\vw$, which collectively contribute a degree of $H=3$ to the network output.
\end{itemize}

\begin{figure}[t]
    \centering
    \subfloat{\includegraphics[width=0.8\linewidth]{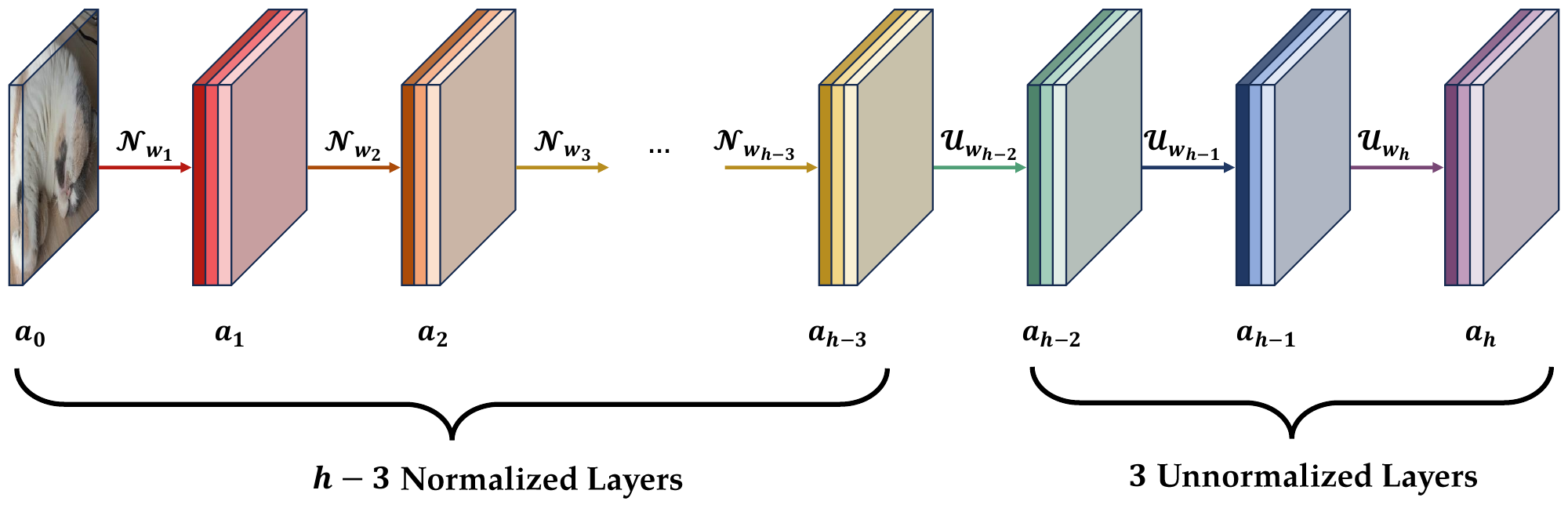}}
    \vspace{1em}
    \subfloat{\includegraphics[width=0.8\linewidth]{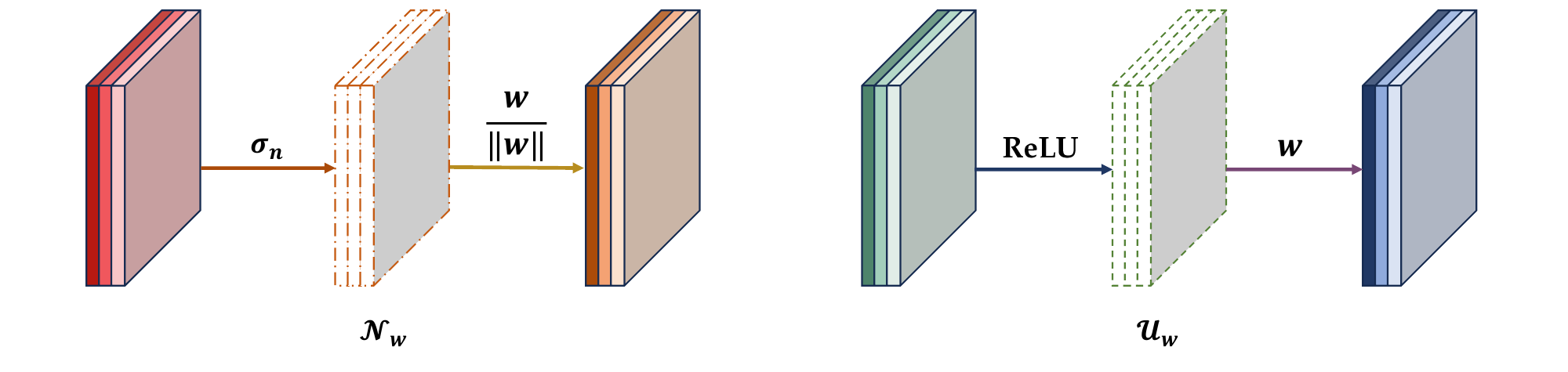}}
    \caption{An instance of construction for $\homo=3$. Following the rules in Proposition~\ref{prop: homo-property}, normalized layers (left) accomodate arbitrary activations and connections with $0$-homogeneity, while the subsequent unnormalized layers (right) compose a $3$-layer MLP to achieve the desired three-homogeneity.}
    \label{fig: illustration of construction}
\end{figure}

\section{Theoretical Results}
In this appendix, we organize the theoretical results as follows: 
Section~\ref{app:notations} details the notations used throughout the paper. 
Section~\ref{sec: proof sketch} presents the proofs and further discussions of the main results in Theorem~\ref{thm: homo leads to generalization}. 
Section~\ref{subsec:proof_cor_5_1} provides the proof for the relaxation of the loss form as shown in Corollary~\ref{cor: extension of loss}. 
Section~\ref{appendix: proof of H-homo} establishes the foundational analysis of the effective stepsize in Lemma~\ref{lem: L-homogeneous property}. 
Section~\ref{app:seperation} proves the optimization separation result between different stepsizes in Theorem~\ref{thm: separation}. 
Section~\ref{app:technical_lemmas} analyzes the weight norm iteration under both Lipschitz and non-Lipschitz conditions. 
Finally, Section~\ref{app:proof_refined} provides the detailed proofs for the non-Lipschitz results discussed in Section~\ref{sec:relax-lipschitz}.

\subsection{Notations}
\label{app:notations}
\begin{table}[htbp]
    \centering
    \caption{Summary of Notations}
    \label{tab:notations}
    \small
    \renewcommand{\arraystretch}{1.2}
    \begin{tabularx}{\textwidth}{@{}l X@{}}
        \toprule
        \textbf{Symbol} & \textbf{Description} \\
        \midrule
        \multicolumn{2}{l}{\textit{\textbf{Model}}} \\ 
        $\Phi(\vw; \X)$ & Neural network function $\Phi(\cdot; \cdot)$ with weights $\vw$ and input $\X$ \\
        $f_\vw(\cdot)$ & Equivalent notation to $\Phi(\vw; \cdot)$ \\
        $\homo$ & Degree of positive homogeneity of the neural network $\Phi$ \\
        $\vw_t$ & Weight vector at iteration $t$ \\
        $\vv_t$ & Direction of weight (normalized weight), defined as $\vv_t = \vw_t / \|\vw_t\|$ \\
        \midrule
        \multicolumn{2}{l}{\textit{\textbf{Optimization}}} \\
        $S$ & Training dataset, a realization of $\cP$ \\ 
        $\eta_t$ & Stepsize (learning rate) in SGD updates \\ 
        $\tilde{\eta}_t$ & Effective stepsize governing dynamics on the unit sphere, $\tilde{\eta}_t = \rho_t \|\vw_t\|^{H-2}\eta_t$ \\
        $\rho_t$ & Gradient alignment ratio, defined as $\rho_t = \ell'_t(\Phi(\vw_t)) /{\ell'_t(\Phi(\vv_t))}$ \\
        $T$ & Total number of training iterations \\
        $n$ & Sample size of the training dataset \\
        $[n]$ & The set $\{1, 2, \cdots, n\}$ \\
        \midrule
        \multicolumn{2}{l}{\textit{\textbf{Loss and Generalization}}} \\
        $\cP$ & Joint distribution of feature-response pair \\
        $\cD, \cA$ & Distribution of datasets and randomized algorithm \\
        $\bE$ & Expectation of the random variable following it, taken over its subscript \\
        $\ell(y, \hat{y})$ & Individual loss function $\ell(\cdot, \cdot)$ with label $y$ and prediction $\hat{y}$ \\
        $\ell(\vw; \X, \Y)$ & Individual loss function equivalent to $\ell(\Y, \Phi(\vw;\X))$ \\
        $\ell_t(z)$ & Equivalent to $\ell(\Y_t, z)$ where $z$ is the prediction \\
        $\cL(\vw)$ & Population loss, $\cL(\vw) \triangleq \bE_{(\X, \Y) \sim \cP} \ell(\vw; \X, \Y)$ \\
        $\hat{\cL}(\vw)$ & Empirical loss, $\hat{\cL}(\vw) \triangleq \frac{1}{n} \sum_{i \in [n]} \ell(\vw; \X_i, \Y_i)$ \\
        $\cL^s(\vw), \hat{\cL}^s(\vw)$ & Population and empirical risks evaluated on the unit sphere, i.e., $\cL^s(\vw) \triangleq \cL(\vv)$ and $\hat{\cL}^s(\vw) \triangleq \hat{\cL}(\vv)$ \\ 
        \midrule
        \multicolumn{2}{l}{\textit{\textbf{Constants and Bounds}}} \\
        $\bar{\sigma}^2$ & Upper bound of the individual loss on the unit sphere \\
        $\underline{\sigma}^2$ & Bayesian optimal value of the given problem on the unit sphere, $\underline{\sigma}^2 \triangleq \min_{\vw} \cL(\vw/\|\vw\|) \neq 0$ \\
        $\gamma, \beta$ & Constants of $(\gamma, \beta)$-approximate smoothness for individual loss\\ 
        $\tilde{\gamma}$ & Additional rate multiplier of generalization bound, $\tilde{\gamma} = \max\{\gamma, 1/n\}$ \\
        $L$ & Lipschitz constant of individual loss\\ 
        $m_1$ & Derived constant governing the generalization, $m_1 = \frac{4 c_2k_1}{\underline{\sigma}^2} \left[ k_1 k_2\bar{\sigma}^2 + \frac{\beta}{H} \right]$, where $k_1=k_2=1$ in Theorem~\ref{thm: homo leads to generalization} \\
        $c_2$ & Constants specifying the stepsize \\ 
        \bottomrule
    \end{tabularx}
    \normalsize
\end{table}
This paper uses the following standard notations for asymptotic analysis. For any positive functions $f(n)$ and $g(n)$, let $C > 0$ denote a constant independent of $n$:
\begin{itemize}
    \item $f(n) = \Theta(g(n))$ if $\lim_{n \to \infty} \frac{f(n)}{g(n)} = C$;
    \item $f(n) = \cO(g(n))$ if $\limsup_{n \to \infty} \frac{f(n)}{g(n)} < \infty$;
    \item $f(n) = \Omega(g(n))$ if $\liminf_{n \to \infty} \frac{f(n)}{g(n)} > 0$;
    \item $f(n) = o(g(n))$ if $\lim_{n \to \infty} \frac{f(n)}{g(n)} = 0$;
    \item $f(n) = \omega(g(n))$ if $\lim_{n \to \infty} \frac{f(n)}{g(n)} = \infty$.
\end{itemize}
A comprehensive summary of the other notations used throughout this paper is provided in Table~\ref{tab:notations}.

\subsection{Proof of Theorem~\ref{thm: homo leads to generalization}}
\label{sec: proof sketch}
This section derives the results in Theorem~\ref{thm: homo leads to generalization}. 
The outline of our approach unfolds as follows: 

First, we establish the relationship between the stepsize $\eta_t$ and the effective stepsize $\tilde{\eta}_t$ on the unit sphere. 
Utilizing Lemma~\ref{lem: L-homogeneous property}, we analyze how different orders of stepsizes induce varying orders of effective stepsizes under the homogeneous setting. 
Building on this, we strategically select the stepsize regime to ensure the effective stepsize remains sufficiently small to guarantee generalization. 
We refer to Lemma~\ref{lem: (effective) stepsize} for more details. 

The remaining part requires deriving the algorithmic stability under such effective stepsize. 
Directly achieving stability proves elusive due to the presence of $\|\vw_{t+1}\|$ and $\|\vw_{t}\|$ within the updating rule stipulated in \eqref{eqn: updating rule}. 
Consequently, bounding the ratio $\|\vw_{t} \| / \|\vw_{t+1}\|$ emerges as a requirement. 
This imperative stems from the upper bound imposed on the loss function by Assumption~\ref{assump: loss bound}.
We summarize the discussion in Lemma~\ref{lem: generalization error initial}. 

\subsubsection{(Effective) Stepsize}
\begin{lemma}[(Effective) stepsize]
\label{lem: (effective) stepsize}
Assume $\bE_\cA \| \vw_0 \|^2 = 1$ and $\homo > 2$ without loss of generality. 
Under Assumption~\ref{assump: loss bound} and Assumption~\ref{assump: approximate smooth and lipshitz}:
\begin{enumerate}
    \item there exist infinitely many stepsizes satisfying $\bE_{\cA, \cD} \eta_t = \Omega(t^{\frac{\homo-4}{2}})$, such that the effective stepsize is of order $\Theta(1 / t)$,
    \item there exist infinitely many stepsizes satisfying $\bE_{\cA, \cD} \eta_t \le \frac{2c_2}{\homo \underline{\sigma}^2} \left(t+1\right)^\frac{\homo-4}{2}=\cO(t^{\frac{\homo-4}{2}})$ with $c_2 \geq 1$, such that the effective stepsize is of order $\cO\left(1/t\right)$. 
    \item there exist infinitely many stepsizes satisfying $\bE_{\cA, \cD} \eta_t = o(t^{\frac{H-4}{2}})$, such that the effective stepsize is of order $o(1/t)$. 
\end{enumerate}
where the expectation is taken over the randomness of the dataset and the algorithm. 
\end{lemma}
Lemma~\ref{lem: (effective) stepsize} is the core of our proof, which implies that while one may opt for a stepsize of the order $\Omega(t^{\frac{\homo - 4}{2}})$, its manifestation as an effective stepsize upon projection onto a sphere can be of order $\Theta(1/t)$. 
This characteristic inherently holds the potential for establishing algorithmic stability. 
We remark that the existence arguments here are general since one could set $\tilde{\eta}_t = \Theta(1/t)$ with different constants.
Hence, even in the context of existence arguments, there persists a degree of freedom attributed to constants.
{
\textbf{Discussion on Stochasticity and Stepsize Regimes.}
We clarify the stochastic nature of the relationship $\tilde{\eta}_t = \eta_t \|\vw_t\|^{\homo-2}$, which distinguishes our regime from standard SGD. 
In the SGD setting, the weight norm $\|\vw_t\|$ is a random variable dependent on the sampling trajectory. 
Consequently, $\eta_t$ and $\tilde{\eta}_t$ cannot simultaneously be deterministic unless $\|\vw_t\|$ remains constant, which generally does not hold. 

Three scenarios arise: 
(1) $\eta_t$ is deterministic, rendering $\tilde{\eta}_t$ random; 
(2) $\tilde{\eta}_t$ is deterministic, which requires $\eta_t$ to be a random variable; 
or (3) both are random variables. 
Our analysis adopts Scenario (2). 
Specifically, we impose a deterministic decay on the effective stepsize $\tilde{\eta}_t$ (\eg, $\Theta(1/t)$), to ensure simple and stable update rule on the unit sphere. 

Therefore, the \emph{existence} established in our statements is twofold: it asserts the existence of random variables $\eta_t$ capable of rendering the effective stepsize (a) deterministic, and (b) falling within the regime where it, as a non-random variable, satisfies the stated asymptotic bounds. 
}

\textbf{Discussion on $\log(T)$ convergence. }
We acknowledge the work by \citet{DBLP:conf/iclr/SoudryHNS18}, which similarly explores a classification setting with direction convergence.
Different from the $\sqrt{T}$ rate proposed in Lemma~\ref{lem: (effective) stepsize}, the weight norm in \citet{DBLP:conf/iclr/SoudryHNS18} demonstrates an approximate $\log T$ growth.
This distinction arises due to variations in loss forms, leading to non-conflicting conclusions. 
It's worth noting that deriving an effective stepsize under their particular regime presents considerable complexity.

\begin{proof}[Proof of Lemma~\ref{lem: (effective) stepsize}]
We begin by proving the \emph{first} statement of Lemma~\ref{lem: (effective) stepsize}. 
By Lemma~\ref{lem: weight norm iteration}, if $\tilde{\eta}_t \leq \frac{\homo \underline{\sigma}^2}{2 L^2}$, we have: 
    \begin{equation*}
        \bE_{\cA, \cD} \| \vw_{t+1} \|^2 \leq \bE_{\cA, \cD} \|\vw_{t} \|^2 \left( 1 - \frac{1}{2} \homo \underline{\sigma}^2 \tilde{\eta}_t \right),
    \end{equation*}
Therefore, WLOG\footnote{Setting $\tilde{\eta}_t = \frac{2 c_2}{\homo \underline{\sigma}^2(t+1)}$ with $c_2 \geq 1$ may lead to the requirement that $2 L \leq \homo \underline{\sigma}^2$ to satisfy the condition $\tilde{\eta}_t \leq \frac{\homo \underline{\sigma}^2}{2 L^2}$ for all $t \geq 0$. However, this is not strictly necessary. One can always set $\tilde{\eta}_t = \frac{2 c_2}{\underline{\sigma}^2(t+U)}$ for a sufficiently large $U$, which satisfies the requirement $\tilde{\eta}_t \leq \frac{\homo \underline{\sigma}^2}{2 L^2}$ without influencing the asymptotic order.} setting $\tilde{\eta}_t = \frac{2 c_2}{\homo \underline{\sigma}^2(t+1)}$ with $c_2 \geq 1$ leads to
\begin{equation}\label{eqn: lemma stepsize 2}
    \bE_{\cA, \cD} \| \vw_{t} \|^2 \leq \bE_{\cA, \cD} \| \vw_{t-1} \|^2 (1-1/{(t+1)}) \leq 1/{(t+1)} \bE_{\cA} \| \vw_{0} \|^2 = 1/{(t+1)}.
\end{equation}
We derive the order of the stepsize given the choice of the effective stepsize. 
Notice that based on the order of the weight and the effective stepsize, the stepsize satisfies
\begin{equation*}
\begin{aligned}
    \bE_{\cA, \cD} \eta_t 
    &\overset{(i)}{=} \bE_{\cA, \cD} \frac{\tilde{\eta}_t}{\| \vw_t\|^{\homo - 2}} \\
    &\overset{(ii)}{\geq} \frac{\tilde{\eta}_t}{[\bE_{\cA, \cD} \| \vw_t\|^{2}]^{\frac{\homo - 2}{2}}} \\
    &\overset{(iii)}{\geq} \frac{2 (t+1)^{(\homo - 2)/2}}{\homo \underline{\sigma}^2 (t+1)} = \frac{2}{\homo \underline{\sigma}^2}(t+1)^{(\homo - 4)/2}=\Omega(t^{(\homo - 4)/2}),
\end{aligned}
\end{equation*}
where the equation $(i)$ is due to the definition of $\tilde{\eta}_t$, the inequality $(ii)$ is from Jensen's inequality, and the inequality $(iii)$ is based on \eqref{eqn: lemma stepsize 2}. 
Specifically, by taking $f(x) = x^{-(\homo - 2)/2}$ which is convex when $x>0$ and $\homo > 2$, we get from Jensen's inequality that $\bE f(x) \geq f(\bE x)$. 
Therefore, setting $x = \| \vw\|^2$ leads to $\bE \| \vw\|^{-(\homo - 2)} \geq [{\bE \| \vw\|^2}]^{-(\homo - 2)/2}$. 

We next prove the \emph{second} statement of Lemma~\ref{lem: (effective) stepsize} by contradiction. 
In this case, it still holds for the weight norm that: 
\begin{equation*}
    \bE_{\cA, \cD} \| \vw_{t+1} \|^2 \leq \bE_{\cA, \cD} \|\vw_{t} \|^2 \left( 1 - \frac{1}{2} \homo \underline{\sigma}^2 \tilde{\eta}_t \right) \leq \bE_{\cA} \|\vw_{0} \|^2 \exp\left(- \frac{1}{2} \homo \underline{\sigma}^2 \sum_{k=0}^t \tilde{\eta}_k\right).
\end{equation*}
Suppose, for the sake of contradiction, that a stepsize satisfying $\bE_{\cA, \cD}\eta_t \le \frac{2c_2}{\homo \underline{\sigma}^2} \left(t+1\right)^\frac{\homo-4}{2}=\cO(t^\frac{\homo-4}{2})$ could lead to a larger effective stepsize $\tilde{\eta}_t > \frac{2 c_2}{\homo \underline{\sigma}^2}\cdot\frac{1}{t+1}$. 
Given this assumption on $\tilde{\eta}_t$, the summation in the exponent is lower bounded by: 
\begin{equation*}
    \sum_{k=0}^t \tilde{\eta}_k > \frac{2 c_2}{\homo \underline{\sigma}^2}\sum_{k=0}^t\frac{1}{k+1}\ge\left. \ln(k+1)\right|_0^{t+1}=\frac{2 c_2}{\homo \underline{\sigma}^2}\ln(t+2). 
\end{equation*}
Consequently, the weight norm would decay at a rate of $\bE_{\cA, \cD} \| \vw_{t+1} \|^2 < \frac{1}{t+2} = \cO(1/t)$.
This rapid decay implies that the stepsize satisfies:
\begin{equation*}
    \bE_{\cA, \cD} \eta_{t+1} > \frac{\frac{2 c_2}{\homo \underline{\sigma}^2}\cdot\frac{1}{t+2}}{[1/(t+2)]^{(\homo - 2)/2}} = \frac{2 c_2}{\homo \underline{\sigma}^2} \cdot\frac{(t+2)^{(\homo - 2)/2}}{t+2}  = \frac{2 c_2}{\homo \underline{\sigma}^2} \cdot (t+2)^{(\homo - 4)/2}. 
\end{equation*}
which contradicts the initial hypothesis that $\bE_{\cA, \cD} \eta_t \le \frac{2c_2}{\homo \underline{\sigma}^2} \left(t+1\right)^\frac{\homo-4}{2}$. 

Therefore, given the stepsize $\bE_{\cA, \cD} \eta_t \le \frac{2c_2}{\homo \underline{\sigma}^2} \left(t+1\right)^\frac{\homo-4}{2}$, the deterministic effective stepsize must satisfy $\tilde{\eta}_t = \cO(1/t)$. 

Finally, we prove the \emph{third} statement of Lemma~\ref{lem: (effective) stepsize}. 
According to Lemma~\ref{lem: norm iterate}, for any given sampling trajectory, the weight norm satisfies:
\begin{equation*}
\|\vw_{t+1}\|^2 = \|\vw_t\|^2 (1 - 2H \tilde{\eta}_t \ell(\vv_t) + \tilde{\eta}_t^2 \|\nabla \ell(\vv_t)\|^2) \ge \|\vw_t\|^2 (1 - 2H \bar{\sigma}^2 \tilde{\eta}_t).
\end{equation*} 
Define an auxiliary sequence $r_t \triangleq t \|\vw_t\|^2$. 
Since $\tilde{\eta}_t = o(1/t)$, the iteration of $r_t$ is governed by: 
\begin{equation*}
\frac{r_{t+1}}{r_t} = \frac{t+1}{t} \cdot \frac{\|\vw_{t+1}\|^2}{\|\vw_t\|^2} \ge \left( 1 + \frac{1}{t} \right) (1 - 2H \bar{\sigma}^2 \tilde{\eta}_t) = 1 + \frac{1}{t} - 2H \bar{\sigma}^2 \tilde{\eta}_t + o(\frac{1}{t^2}).
\end{equation*} 
Then, there exists a $t_0$ such that for all $t > t_0$, the term $2H \bar{\sigma}^2 \tilde{\eta}_t < \frac{1}{2t}$. 
It holds that:
\begin{equation}
\label{eq:rt}
\frac{r_{t+1}}{r_t} > 1 + \frac{1}{2t}.
\end{equation}
Iterating \eqref{eq:rt} from $t_0$ to $T-1$, we obtain:
\begin{equation*} 
r_T \ge r_{t_0} \prod_{t=t_0}^{T-1} (1 + \frac{1}{2t}) \ge r_{t_0}\left(1+\sum_{t=t_0}^{T-1}\frac{1}{2t}\right).
\end{equation*}
The divergence of $\sum \frac{1}{t}$ implies $\lim_{T\to\infty} r_T = \infty$, which entails $\|\vw_T\|^2 = \omega(T^{-1})$. 
Recall that $\tilde{\eta}_t \triangleq  \|\vw_t\|^{\homo-2}\eta_t$ from \eqref{eqn: updating rule}. 
Substituting the asymptotic order into this definition yields:
\begin{equation*}
\eta_t = \frac{\tilde{\eta}_t}{[\|\vw_t\|^2]^{\frac{\homo-2}{2}}} \\
=  \frac{o(1/t)}{[\omega(1/t)]^{\frac{\homo-2}{2}}} \\
= o(t^{\frac{\homo-4}{2}}). 
\end{equation*}
Notably, this upper bound holds with $t_0$ independent of sampling randomness. 
Taking the expectation yields $\bE_{\cA, \cD}\eta_t=o(t^{\frac{\homo-4}{2}})$. 
This completes the proof. 
\end{proof}
\subsubsection{Generalization Gap}
We next show in Lemma~\ref{lem: generalization error initial} that under the effective stepsize $\tilde{\eta}_t = \Theta(1/t)$, the generalization gap (or, algorithmic stability) can be bounded. 
Notably, the generalization gap is non-decreasing with respect to the effective stepsize sequence. 

\begin{lemma}[Generalization performance under effective stepsize]
\label{lem: generalization error initial}
Under the effective stepsize $\tilde{\eta}_t = \frac{2 c_2}{\homo\underline{\sigma}^2(t+2)}$ with $c_2 \geq 1$, assuming that the function is $L$-Lipschitz and $\gamma$-$\beta$ approximately smooth, it holds that
\begin{equation*}
\bE_{\cA, \cD} \cL(\vv_t) - \cL_n(\vv_t) \leq   (L m_2)^{\frac{1}{m_1 + 1}}\left[1 + \frac{1}{m_1}\right] \frac{T^\frac{m_1}{m_1 + 1}}{n},
\end{equation*}
where $m_1 = \frac{2c}{\underline{\sigma}^2} [\bar{\sigma}^2 + \frac{\beta}{\homo}]$, and $m_2 = \frac{4c}{\homo \underline{\sigma}^2} (L + \gamma n)$.
\end{lemma}

\begin{proof}[Proof of Lemma~\ref{lem: generalization error initial}]
Our results are partly inspired by the results in \citet{DBLP:conf/icml/HardtRS16}.
We focus on weight directions $\vv_t$ and $\vv^\prime_t$ which are trained on dataset differing in at most one sample over $t$ iterations. To bound the algorithmic stability, the core is to bound the difference between $\vv_{t+1}$ and $\vv^\prime_{t+1}$, which is denoted by $\Delta(\vv_{t+1}, \vv^\prime_{t+1}) = \| \vv_{t+1} - \vv^\prime_{t+1} \|$.
To do so, we first denote the weight after one iteration by $\bar{\vv}_{t+1} = \vv_t - \tilde{\eta}_t \nabla_{\vv_t} \ell_t(\vv_t)$. 

When choosing the same sample during the training,
\begin{equation}\label{eqn: lemma gen 1}
    \begin{split}
        &\Delta(\bar{\vv}_{t+1}, \bar{\vv}^\prime_{t+1}) \\
        = & \| \bar{\vv}_{t+1} - \bar{\vv}^\prime_{t+1} \| \\
        = & \|( \vv_t - \tilde{\eta}_t \nabla_{\vv_t} \ell_t(\vv_t) )- ( \vv_t^\prime - \tilde{\eta}_t \nabla_{\vv_t^\prime} \ell_t(\vv_t^\prime) ) \| \\ 
        \leq & \|\vv_t -  \vv_t^\prime \| + \tilde{\eta}_t \|\nabla_{\vv_t} \ell_t(\vv_t) - \nabla_{\vv_t^\prime} \ell_t(\vv_t^\prime) \|.
    \end{split}
\end{equation}
Note that the gradient difference can be bounded by Assumption~\ref{assump: approximate smooth and lipshitz} using a $\beta$-smooth function $\bar{\ell}$, then we have that
\begin{equation}\label{eqn: lemma gen 2}
    \begin{split}
        & \|\nabla_{\vv_t} \ell_t(\vv_t) - \nabla_{\vv_t^\prime} \ell_t(\vv_t^\prime) \| \\ 
        = & \|\nabla_{\vv_t} \ell_t(\vv_t) - \nabla_{\vv_t} \bar{\ell}_t(\vv_t) + \nabla_{\vv_t}\bar{\ell}_t(\vv_t) - \nabla_{\vv_t^\prime} \bar{\ell}_t(\vv_t^\prime) + \nabla_{\vv_t^\prime} \bar{\ell}_t(\vv_t^\prime)- \nabla_{\vv_t^\prime} \ell_t(\vv_t^\prime) \| \\
        \leq & \|\nabla_{\vv_t} \ell_t(\vv_t) - \nabla_{\vv_t} \bar{\ell}_t(\vv_t) \|  + \| \nabla_{\vv_t}\bar{\ell}_t(\vv_t) - \nabla_{\vv_t^\prime} \bar{\ell}_t(\vv_t^\prime) \| + \| \nabla_{\vv_t^\prime} \bar{\ell}_t(\vv_t^\prime)- \nabla_{\vv_t^\prime} \ell_t(\vv_t^\prime) \| \\
        \leq & 2 \gamma + \beta \|\vv_t -  \vv_t^\prime \|.
    \end{split}
\end{equation}
By plugging \eqref{eqn: lemma gen 1} into \eqref{eqn: lemma gen 2}, we obtain
\begin{equation}\label{eqn: lemma gen 3}
    \Delta(\bar{\vv}_{t+1}, \bar{\vv}^\prime_{t+1}) \leq (1 + \tilde{\eta}_t \beta) \Delta(\bar{\vv}_{t}, \bar{\vv}^\prime_{t}) + 2 \tilde{\eta}_t \gamma.
\end{equation}

When choosing a different sample during the training, it similarly holds that
\begin{equation}\label{eqn: lemma gen 4}
    \begin{split}
        &\Delta(\bar{\vv}_{t+1}, \bar{\vv}^\prime_{t+1}) \\
        = & \| \bar{\vv}_{t+1} - \bar{\vv}^\prime_{t+1} \| \\
        = & \|( \vv_t - \tilde{\eta}_t \nabla_{\vv_t} \ell_t(\vv_t) )- ( \vv_t^\prime - \tilde{\eta}_t \nabla_{\vv_t^\prime} \ell_t(\vv_t^\prime) ) \| \\ 
        \leq & \|\vv_t -  \vv_t^\prime \| + \tilde{\eta}_t \|\nabla_{\vv_t} \ell_t(\vv_t) - \nabla_{\vv_t^\prime} \ell_t(\vv_t^\prime) \| \\
        \leq & \Delta(\bar{\vv}_{t}, \bar{\vv}^\prime_{t}) + 2 \tilde{\eta}_t L.
    \end{split}
\end{equation}

Due to the randomness of SGD, where with probability $1 - 1/n$ the same sample is chosen and with probability $1/n$ a different sample is selected, based on \eqref{eqn: lemma gen 3} and \eqref{eqn: lemma gen 4}, it holds that
\begin{equation*}
    \begin{split}
        &\bE_{\cA, \cD} \Delta(\bar{\vv}_{t+1}, \bar{\vv}^\prime_{t+1}) \\
        \leq & \bE_{\cA, \cD} (1 - \frac{1}{n} ) \left[(1 + \tilde{\eta}_t \beta) \Delta(\bar{\vv}_{t}, \bar{\vv}^\prime_{t}) + 2 \tilde{\eta}_t \gamma\right] + \frac{1}{n} [\Delta({\vv}_{t}, {\vv}^\prime_{t}) + 2 \tilde{\eta}_t L] \\
        = & \bE_{\cA, \cD} \left(1 + \tilde{\eta}_t \beta(1 - \frac{1}{n}) \right) \Delta({\vv}_{t}, {\vv}^\prime_{t}) + \frac{2 \tilde{\eta}_t L}{n} + (1 - \frac{1}{n} )2 \tilde{\eta}_t \gamma\\
        \leq & \bE_{\cA, \cD} (1 + \tilde{\eta}_t \beta) \Delta({\vv}_{t}, {\vv}^\prime_{t}) + \frac{2 \tilde{\eta}_t L}{n} + 2 \tilde{\eta}_t \gamma.
    \end{split}
\end{equation*}

By the results in Lemma~\ref{lem: norm ratio} and Lemma~\ref{lem: relationship between its projection}, we conclude that 
\begin{equation}
\label{eq:norm ratio bound}
    \begin{split}
        \Delta(\vv_{t+1}, \vv^\prime_{t+1}) 
        \leq  \max\left\{ \frac{\|\vw_t \|}{\|\vw_{t+1} \|}, \frac{\|\vw_t^\prime \|}{\|\vw_{t+1}^\prime \|} \right\} \Delta(\bar{\vv}_{t+1}, \bar{\vv}^\prime_{t+1}) 
        \leq  \left[1 + 4\frac{c_2 \bar{\sigma}^2}{\underline{\sigma}^2} \frac{1}{t+2}\right] \Delta(\bar{\vv}_{t+1}, \bar{\vv}^\prime_{t+1}).
    \end{split}
\end{equation}
When $t \geq T_0 =  8 \frac{c_2 \bar{\sigma}^2}{\underline{\sigma}^2} = \cO(1)$, taking expectations over the algorithms and dataset and plugging into $\tilde{\eta}_t = \frac{2 c_2}{\homo \underline{\sigma}^2 (t+2)}$ leads to 
\begin{equation}\label{eqn: lemma gen 5}
    \begin{split}
        &\bE_{\cA, \cD} \Delta(\vv_{t+1}, \vv^\prime_{t+1}) \\
        \leq & \left[1 + 4\frac{c_2 \bar{\sigma}^2}{\underline{\sigma}^2} \frac{1}{t+2}\right] \bE_{\cA, \cD} \Delta(\bar{\vv}_{t+1}, \bar{\vv}^\prime_{t+1}) \\
        \leq & \left[1 + 4\frac{c_2 \bar{\sigma}^2}{\underline{\sigma}^2} \frac{1}{t+2}\right]\left[(1 + \tilde{\eta}_t \beta) \bE_{\cA, \cD} \Delta({\vv}_{t}, {\vv}^\prime_{t}) + \frac{2 \tilde{\eta}_t L}{n} + 2 \tilde{\eta}_t \gamma\right] \\
        = & \left[1 + 4\frac{c_2 \bar{\sigma}^2}{\underline{\sigma}^2} \frac{1}{t+2}\right] \left[(1 + \frac{2c_2\beta}{\homo\underline{\sigma}^2 (t+2)} ) \bE_{\cA, \cD} \Delta({\vv}_{t}, {\vv}^\prime_{t}) + \frac{4 c_2 L}{n \homo\underline{\sigma}^2 (t+2)} + \frac{4 c_2}{\homo\underline{\sigma}^2 (t+2)} \gamma\right]\\
        \leq & \left[1 + 4\frac{c_2 \bar{\sigma}^2}{\underline{\sigma}^2} \frac{1}{t+2} + \frac{4 c_2\beta}{\homo\underline{\sigma}^2 (t+2)}\right] \bE_{\cA, \cD}\Delta({\vv}_{t}, {\vv}^\prime_{t}) +\frac{8 c_2 L }{n \homo {\underline{\sigma}^2 (t+2)}} +  \frac{8 c_2\gamma}{\homo\underline{\sigma}^2 (t+2)}\\ 
        \leq & \exp\left(m_1 \frac{1}{t}\right) \bE_{\cA, \cD}\Delta({\vv}_{t}, {\vv}^\prime_{t}) + \frac{m_2}{nt},
    \end{split}
\end{equation}
where $\tilde{\eta}_t \leq 1/\beta$, $m_1 = \frac{4 c_2}{\underline{\sigma}^2} [\bar{\sigma}^2 + \frac{\beta}{\homo}]$, and $m_2 = \frac{8c_2}{\homo \underline{\sigma}^2} (L + \gamma n)$.

According to (3.10) in \citet{DBLP:conf/icml/HardtRS16}, for any given $t_0$, generalization gap can be bounded by 
\begin{equation}
\label{eq:technique}
   \bE_{\cA, \cD} \cL^s_T(\vw) - \hat{\cL}^s_T(\vw) \leq \frac{t_0}{n} + L \bE_{\cA, \cD} (\Delta(\vv_{T}, \vv^\prime_{T}) \ | \ \Delta(\vv_{t_0}, \vv^\prime_{t_0}) = 0).
\end{equation}
For the second term, notice that given $\Delta(\vv_{t_0}, \vv^\prime_{t_0}) = 0$ and \eqref{eqn: lemma gen 5}, it holds that 
\begin{equation*}
    \begin{split}
        &\bE_{\cA, \cD} \Delta(\vv_{T}, \vv^\prime_{T}) \\
        \leq & \sum_{t = t_0 + 1}^{T} \{ \prod_{k=t+1}^T \exp(\frac{m_1}{k}) \} \frac{m_2}{nt} \\
        = & \sum_{t = t_0 + 1}^{T} \exp(m_1 \sum_{k=t+1}^T \frac{1}{k}) \frac{m_2}{nt} \\
        \leq & \sum_{t = t_0 + 1}^{T} \exp(m_1 \log(T/t)) \frac{m_2}{nt} \\
        = & \frac{m_2}{n} T^{m_1} \sum_{t = t_0 + 1}^{T} \frac{1}{t^{m_1 + 1}} \\
        \leq & \frac{m_2}{n} \frac{1}{m_1} (\frac{T}{t_0})^{m_1}.
    \end{split}
\end{equation*}
Therefore, one can bound the generalization gap as
\begin{equation*}
    \bE_{\cA, \cD} \cL^s_T(\vw) - \hat{\cL}^s_T(\vw) \leq \frac{t_0}{n} + L \frac{m_2}{n} \frac{1}{m_1} (\frac{T}{t_0})^{m_1}.
\end{equation*}
By choosing\footnote{Here the choice of $t_0$ naturally holds that $t_0 \geq T_0 =  8 \frac{c_2 \bar{\sigma}^2}{\underline{\sigma}^2} = \cO(1).$} $t_0 = (L m_2)^{\frac{1}{m_1 + 1}} T^\frac{m_1}{m_1 + 1}$, the generalization gap is bounded by
\begin{equation*}
    \bE_{\cA, \cD} \cL^s_T(\vw) - \hat{\cL}^s_T(\vw) \leq (L m_2)^{\frac{1}{m_1 + 1}}\left[1 + \frac{1}{m_1}\right] \frac{T^\frac{m_1}{m_1 + 1}}{n}.
\end{equation*}
We finally remark that here it requires $t_0 < T$, which leads to $L m_2 < T$.
Recall that $m_2 = \frac{8c_2}{\homo \underline{\sigma}^2} (L + \gamma n)$ and therefore it suffices to require $\gamma = o(T/n)$.
\end{proof}

\begin{lemma}
    \label{lem: norm ratio}
    Under the Assumptions in Lemma~\ref{lem: generalization error initial}, there exists a $T_0 = 8 \frac{c_2 \bar{\sigma}^2}{\underline{\sigma}^2} = \cO(1)$ such that for $t > T_0$,
    \begin{equation}
        \frac{\| \vw_t \|}{\| \vw_{t+1} \|}  \leq 1 + 4\frac{c_2 \bar{\sigma}^2}{\underline{\sigma}^2} \frac{1}{t+2}.
    \end{equation}
\end{lemma}
\begin{proof}[Proof of Lemma~\ref{lem: norm ratio}]
Due to Argument~3 in Lemma~\ref{lem: L-homogeneous property}, it holds that
\begin{equation*}
    \frac{\| \vw_t \|}{\| \vw_{t+1} \|} =  [1 + \tilde{\eta}_t^2 \|\nabla_{\vv_t} \ell_t(\vv_t) \|^2 - 2 \homo \tilde{\eta}_t \ell_t(\vv_t)]^{-1/2}.
\end{equation*}
By plugging into the effective stepsize $\tilde{\eta}_t = \frac{2 c_2}{\homo \underline{\sigma}^2(t+2)}$ and the loss upper bound in Assumption~\ref{assump: loss bound}, it holds that
\begin{equation*}
    \frac{\| \vw_t \|}{\| \vw_{t+1} \|} \leq \left[1 - 2 \homo \tilde{\eta}_t \ell_t(\vv_t)\right]^{-1/2} \leq \left[1 -  4\frac{c_2 \bar{\sigma}^2}{\underline{\sigma}^2} \frac{1}{t+2}\right]^{-1/2}.
\end{equation*}
Furthermore, due to the fact that $(1-u)^{-1/2} \leq 1+u$ when $u \in [0, 1/2]$,  
\begin{equation*}
    \frac{\| \vw_t \|}{\| \vw_{t+1} \|} \leq 1 + 4\frac{c_2 \bar{\sigma}^2}{\underline{\sigma}^2} \frac{1}{t+2},
\end{equation*}
which is guaranteed by the condition $t > T_0 = 8 \frac{c_2 \bar{\sigma}^2}{\underline{\sigma}^2}$.
\end{proof}

\begin{lemma}
\label{lem: relationship between its projection}
    Under the Assumptions in Lemma~\ref{lem: generalization error initial}, it holds that 
    \begin{equation}
    \label{eqn: relationship between its projection}
    \Delta(\vv_{t+1}, \vv^\prime_{t+1}) \leq \max\left\{ \frac{\|\vw_t \|}{\|\vw_{t+1} \|}, \frac{\|\vw_t^\prime \|}{\|\vw_{t+1}^\prime \|} \right\} \Delta(\bar{\vv}_{t+1}, \bar{\vv}^\prime_{t+1}).
\end{equation}
\end{lemma}

\begin{proof}[Proof of Lemma~\ref{lem: relationship between its projection}]
We first notice that ${\vv}_{t+1}$ is the projection of $\bar{\vv}_{t+1}$ to a standard sphere.
Therefore, due to the iteration in \eqref{eqn: updating rule}, we derive that 
\begin{equation*}
    {\vv}_{t+1} = \frac{\| \vw_t \|}{ \| \vw_{t+1}\|} \bar{\vv}_{t+1}.
\end{equation*}

We next show the results in \eqref{eqn: relationship between its projection}.
Due to the symmetry, we assume without loss of generality that $\frac{\|\vw_t \|}{\|\vw_{t+1} \|}\leq \frac{\|\vw_t^\prime \|}{\|\vw_{t+1}^\prime \|}$, which indicates that $\| \bar{\vv}^\prime_{t+1} \| \leq \| \bar{\vv}_{t+1} \|$ since $\| \vv_{t+1}\| = \| \vv^\prime_{t+1}\| =1$.
Therefore, it suffices to show that $\Delta(\vv_{t+1}, \vv^\prime_{t+1}) \leq \frac{\|\vw_t^\prime \|}{\|\vw_{t+1}^\prime \|} \Delta(\bar{\vv}_{t+1}, \bar{\vv}^\prime_{t+1}).$
\begin{figure*}[htbp]
  \centering
  \begin{tikzpicture}[scale=0.4]
  \label{fig: illustration of proof}
    \coordinate[label=below:$\vv_{t+1}$] (A) at (-2,0);
    \coordinate[label=below:$\vv_{t+1}^\prime$] (B) at (2,0);
    \coordinate[label=above:$O$] (C) at (0,4);
    \coordinate[label=left:$\bar{\vv}_{t+1}$] (D) at (-4/3,4/3);
    \coordinate[label=left:$\omega \bar{\vv}_{t+1}$] (E) at (-1,2);
    \coordinate[label=right:$\bar{\vv}_{t+1}^\prime$] (F) at (1,2);

    \draw (A) -- (B) -- (C) -- cycle;
    \draw[dashed] (E) -- (F);
    \filldraw[black] (A) circle (2pt);
    \filldraw[black] (B) circle (2pt);
    \filldraw[black] (C) circle (2pt);
    \filldraw[black] (D) circle (2pt);
    \filldraw[black] (E) circle (2pt);
    \filldraw[black] (F) circle (2pt);
    \draw (C) circle ({sqrt{20}});
  \end{tikzpicture}
\end{figure*}

We next only focus on the space spanned by vector $\vv_{t+1}$ and $\vv_{t+1}^\prime$. 
Let $\omega = \frac{\|\bar{\vv}^\prime_{t+1} \|}{\|\bar{\vv}_{t+1}\|}$ denotes a ratio, which guarantees that $\| \omega \bar{\vv}_{t+1}\| = \|\bar{\vv}^\prime_{t+1}  \|$.
 When $\| \bar{\vv}^\prime_{t+1} \| \leq \| \bar{\vv}_{t+1} \|$, we have 
$\Delta(\omega \bar{\vv}_{t+1}, \bar{\vv}_{t+1}^\prime) \leq \Delta( \bar{\vv}_{t+1}, \bar{\vv}_{t+1}^\prime) $.
Therefore, it holds that
\begin{equation*}
\begin{split}
    \Delta(\vv_{t+1}, \vv^\prime_{t+1}) = \frac{\|\vw_t^\prime \|}{\|\vw_{t+1}^\prime \|} \Delta(\omega \bar{\vv}_{t+1}, \bar{\vv}_{t+1}^\prime)  \leq  \frac{\|\vw_t^\prime \|}{\|\vw_{t+1}^\prime \|} \Delta( \bar{\vv}_{t+1}, \bar{\vv}_{t+1}^\prime),
\end{split}
\end{equation*}
where the first equation is due to the projection. 
This completes the proof. 
\end{proof}

\subsection{Proof of Corollary~\ref{cor: extension of loss}}
\label{subsec:proof_cor_5_1}
\begin{proof} 
For clarity, this proof follows the derivation of the second statement in Lemma~\ref{lem: (effective) stepsize}. 
The rest follows from a similar argument. 
The proof proceeds by substituting specific constants with new values while reusing the analytical framework established in the proof of Theorem~\ref{thm: homo leads to generalization}. 

Let $c_2' = k_1 c_2$.  
We first show that for the identical stepsizes satisfying $\bE_{\cA, \cD} \eta_t \le \frac{2c_2}{\homo \underline{\sigma}^2} (t+1)^\frac{\homo-4}{2}$, the effective stepsize in Corollary~\ref{cor: extension of loss} is bounded by:
\begin{equation}
\label{eq:trans_effective_stepsize}
    \tilde{\eta}_t \le \frac{2 c_2'}{\homo \underline{\sigma}^2(t+1)}.
\end{equation}

Suppose, for the sake of contradiction, there exists some $t$ such that $\tilde{\eta}_t > \frac{2 c_2'}{\homo \underline{\sigma}^2(t+1)}$. 
We have: 
\begin{align*}
    \bE_{\cA, \cD} \eta_{t+1} &= \bE_{\cA, \cD} \frac{\tilde{\eta}_{t+1}}{\rho_{t+1}\| \vw_{t+1}\|^{\homo - 2}} 
    \geq \frac{\tilde{\eta}_{t+1}}{\rho_{t+1}[\bE_{\cA, \cD} \| \vw_{t+1}\|^{2}]^{(\homo - 2)/2}}
    \geq \frac{1}{k_1} \cdot \frac{\tilde{\eta}_{t+1}}{[\bE_{\cA, \cD} \| \vw_{t+1}\|^{2}]^{(\homo - 2)/2}} \\
    &> \frac{1}{k_1} \cdot \frac{\frac{2 c_2'}{\homo \underline{\sigma}^2 (t+2)}}{(1/(t+2))^{\frac{\homo-2}{2}}}
    = \frac{2 (c_2'/k_1)}{\homo \underline{\sigma}^2} (t+2)^{\frac{\homo-4}{2}}.
\end{align*}
Substituting $c_2'/k_1 = c_2$, we obtain $\bE_{\cA, \cD} \eta_{t+1} > \frac{2 c_2}{\homo \underline{\sigma}^2} (t+2)^{\frac{\homo-4}{2}}$, which contradicts the given condition on the stepsize. 
Thus, the effective stepsize must satisfy $\tilde{\eta}_t \le \frac{2 c_2'}{\homo \underline{\sigma}^2(t+1)}$. 

Next, we bound the norm ratio $\frac{\|\vw_t\|}{\|\vw_{t+1}\|}$. 
Combining the assumption that $\rho_t \le k_1$, $z \ell'_t(z) \le k_2 \ell_t(z)$, and the homogeneity of the neural network $\Phi$, we have:
\begin{equation*}
\label{eq:trans_inner_product}
\begin{split}
    \vw_t \top \nabla_{\vw_t} \ell_t(\vw_t) 
    &= \ell'_t(\Phi(\vw_t)) \vw_t\top \nabla_{\vw_t}\Phi(\vw_t) \\
    &= \rho_t \ell'_t(\Phi(\vv_t)) \vw_t\top \nabla_{\vw_t}\Phi(\vw_t) \\
    &= \rho_t\|\vw_t\|^\homo  \ell'_t(\Phi(\vv_t)) \vv_t\top\nabla_{\vv_t} \Phi(\vv_t) \\
    &= \rho_t \homo \|\vw_t\|^\homo \Phi(\vv_t) \ell'_t(\Phi(\vv_t)) \\
    &\le k_1  \homo \|\vw_t\|^\homo \Phi(\vv_t) \ell'_t(\Phi(\vv_t)) \\
    &\le k_1 k_2 \homo\|\vw_t\|^\homo \ell_t(\Phi(\vv_t)) \\
    &\le k_1 k_2 \bar{\sigma}^2\homo\|\vw_t\|^\homo. 
\end{split}
\end{equation*}
Substituting this bound into the norm iteration $\|\vw_{t+1}\|^2 = \|\vw_t \|^2 + \eta_t^2 \|\nabla_{\vw_t} \ell_t (\vw_t) \|^2 - 2 \eta_t \vw_t^\top \allowbreak \nabla_{\vw_t} \ell_t (\vw_t)$, we obtain:
\begin{align*}
    \|\vw_{t+1}\|^2 
    &\ge \|\vw_t\|^2 - 2\eta_t \left( k_1 k_2 \homo \bar{\sigma}^2 \|\vw_t\|^\homo \right) \\
    &= \|\vw_t\|^2 - 2 \left( \frac{\tilde{\eta}_t}{\|\vw_t\|^{\homo-2}} \right) k_1 k_2 \homo \bar{\sigma}^2 \|\vw_t\|^\homo \\
    &= \|\vw_t\|^2 \left( 1 - 2 k_1 k_2 \homo \bar{\sigma}^2 \tilde{\eta}_t \right).
\end{align*}
Then, we can replace the upper bound in \eqref{eq:norm ratio bound} with $1 + 4\frac{c_2 k_1k_2\bar{\sigma}^2}{\underline{\sigma}^2} \frac{1}{t+2}$. 

As the remaining derivation for the generalization gap in Lemma~\ref{lem: generalization error initial} remains valid, we directly apply the established results by substituting $c_2$ with $c_2' = k_1 c_2$, and $\bar{\sigma}^2$ with $\bar{\sigma}^2_0=k_1 k_2 \bar{\sigma}^2$. 
Finally, updating the definitions of $m_1$ and $m_2$ with these values yields the modified constants stated in Corollary~\ref{cor: extension of loss}.
\end{proof}

The proof of Theorem~\ref{cor: Three-Homogeneous Leads to Generalization} can be directly derived from Theorem~\ref{thm: homo leads to generalization} given $H=3$. 

\subsection{Proof of Lemma~\ref{lem: L-homogeneous property}}
\label{appendix: proof of H-homo}

Lemma~\ref{lem: L-homogeneous property} is a combination of Lemma~\ref{lem: effective lr}, Lemma~\ref{lem: inner product}, and Lemma~\ref{lem: norm iterate}.

\begin{lemma}[Effective stepsize]
\label{lem: effective lr}
For an $\homo$-homogeneous function $\ell_t$ with $H \neq 0$, 
the effective stepsize satisfies $\tilde{\eta}_t = \eta_t \| \vw_t \|^{\homo - 2}$, where the effective stepsize $\tilde{\eta}_t$ is defined as the stepsize in updating the direction $\vv_t = \vw_t / \|\vw_t\|$. 
\end{lemma}

\begin{proof}[Proof of Lemma~\ref{lem: effective lr}]

For an $\homo$-homogeneous function $\ell_t$ with $H\neq0$, it holds that 
\begin{equation}
    \nabla_{\vw_t} \ell_t(\vw_t) = \|\vw_t \|^{H-1} \nabla_{\vv_t} \ell_t(\vv_t),
\end{equation}
where $\vv_t = \vw_t / \|\vw_t \|$. This is due to the fact that $\ell_t(\vw_t) = \| \vw_t\|^{H} \ell_t(\vv_t)$ and plug it into the definition of derivation. 
Therefore, we rewrite the update in \eqref{eqn: iteration} as 
\begin{equation}
\label{eqn: updating rule}
    \vv_{t+1} \|\vw_{t+1} \| = \| \vw_t\| ( \vv_t - \eta_t \|\vw_t \|^{\homo-2} \nabla_{\vv_{t}} \ell_t(\vv_t)).
\end{equation}
This implies that the update on the direction $\vv_t$ has effective stepsize as $\eta_t \|\vw_t \|^{\homo-2} $.
\end{proof}

\begin{lemma}[Inner Product]
\label{lem: inner product}
For an $\homo$-homogeneous function $\ell_t$ with $\homo \neq 0$, it holds that 
\begin{equation}
    \vw_t^{\top} \nabla_{\vw_t} \ell_t(\vw_t) = \homo \ell_t(\vw_t).
\end{equation}
\end{lemma}
\begin{proof}[Proof of Lemma~\ref{lem: inner product}]
    Note that for an $\homo$-homogeneous $\ell_t$, it holds that $\ell_t(c\vw_t) = c^{\homo} \ell_t(\vw_t)$. By taking derivation on $c$, it holds that
    \begin{equation*}
        \vw_t^\top \frac{\nabla \ell_t(c\vw_t)}{\nabla c\vw_t} = \homo c^{\homo-1} \ell_t(\vw_t).
    \end{equation*}
    Taking $c=1$ leads to
    \begin{equation*}
        \vw_t^\top \nabla_{\vw_t} \ell_t(\vw_t) = \homo \ell_t(\vw_t).
    \end{equation*}
\end{proof}

\begin{lemma}[Norm Iterate]
\label{lem: norm iterate}
For an $\homo$-homogeneous function $\ell_t$ with $\homo \neq 0$, it holds that 
\begin{equation}
    \| \vw_{t+1} \|^2 =  \|\vw_t \|^2 ( 1 + \tilde{\eta}_t^2 \|\nabla_{\vv_t} \ell_t(\vv_t) \|^2 - 2 \homo \tilde{\eta}_t \ell_t(\vv_t)).
\end{equation}
\end{lemma}
\begin{proof}[Proof of Lemma~\ref{lem: norm iterate}]
    We consider the iteration of $\| \vw_{t+1} \|$, as follows:
\begin{equation*}
    \begin{split}
        &\| \vw_{t+1} \|^2 \\
        =& \|\vw_t - \eta_t \nabla_{\vw_t} \ell_t (\vw_t) \|^2 \\
        =& \|\vw_t \|^2 + \eta_t^2 \|\nabla_{\vw_t} \ell_t (\vw_t) \|^2 - 2 \eta_t \vw_t^\top \nabla_{\vw_t} \ell_t (\vw_t) \\
        \overset{(i)}{=}&\|\vw_t \|^2 + \eta_t^2 \|\nabla_{\vw_t} \ell_t (\vw_t) \|^2  - 2 \homo \eta_t \ell_t(\vw_t) \\
         \overset{(ii)}{=}& \|\vw_t \|^2 + \eta_t^2 \| \vw_t\|^{2(\homo-1)} \|\nabla_{\vv_t} \ell_t(\vv_t) \|^2 - 2 \homo \eta_t \| \vw_t\|^{\homo} \ell_t(\vv_t) \\
        =& \|\vw_t \|^2 ( 1 + \eta_t^2 \| \vw_t\|^{2(\homo-2)} \|\nabla_{\vv_t} \ell_t(\vv_t) \|^2 - 2 \homo \eta_t \| \vw_t\|^{\homo-2} \ell_t(\vv_t) ) \\
         \overset{(iii)}{=}& \|\vw_t \|^2 ( 1 + \tilde{\eta}_t^2 \|\nabla_{\vv_t} \ell_t(\vv_t) \|^2 - 2 \homo \tilde{\eta}_t \ell_t(\vv_t)).
    \end{split}
\end{equation*}
where the equation $(i)$ is due to Lemma \ref{lem: inner product}, the equation $(ii)$ is based on the fact that $\|\nabla_{\vw_t} \ell_t (\vw_t)\| = \|\vw_t\|^{H-1} \allowbreak \|\nabla_{\vv_t} \ell_t(\vv_t)\|$ and $\ell_t(\vw_t) = \|\vw_t\|^H \ell_t(\vv_t)$, and the equation $(iii)$ follows from the effective stepsize result in Lemma~\ref{lem: effective lr}.
\end{proof}
{
\subsection{Proof of Theorem~\ref{thm: separation}}
\label{app:seperation}
This section provides the proofs of Theorem~\ref{thm: separation}.
This result is based on the following Lemma~\ref{lem: effective stepsize for 1/t} and Lemma~\ref{lem: optimization}.
Lemma~\ref{lem: effective stepsize for 1/t} demonstrate that the effective stepsize would be $\cO(\frac{1}{t\log t})$ given the stepsize $\bE_{\cA, \cD} \eta_t = \cO(1/t)$. 
As a comparison, the effective stepsize used in the first statement of Lemma~\ref{lem: (effective) stepsize} would be $\Theta(1/t)$, given a stepsize $\Omega(1/\sqrt{t})$. 
Lemma~\ref{lem: optimization} further derives the convergence rate given the effective stepsize. 
Combining Lemma~\ref{lem: effective stepsize for 1/t} and Lemma~\ref{lem: optimization} leads to Theorem~\ref{thm: separation}.
}

\begin{lemma}[The Effective Stepsize for Stepsize $\cO(1/t)$]
\label{lem: effective stepsize for 1/t}
Under the settings in Theorem~\ref{thm: separation}, there exist infinitely many stepsizes satisfies $\bE_{\cA, \cD} \eta_t = \cO(1/t)$, such that the effective stepsize satisfies $\tilde{\eta}_t = \cO(\frac{1}{t\log t})$. 
\end{lemma}

\begin{proof}[Proof of Lemma~\ref{lem: effective stepsize for 1/t}]
    Firstly, notice that according to Lemma~\ref{lem: weight norm iteration}, it holds that
\begin{equation*}
    \bE_{\cA, \cD} \| \vw_{t+1} \|^2 \leq \bE_{\cA, \cD} \|\vw_{t} \|^2 \left( 1 - \frac{1}{2}  \homo \underline{\sigma}^2 \tilde{\eta}_t \right) \leq \bE_{\cA, \cD} \|\vw_{t} \|^2 \exp\left(- \frac{1}{2} \homo \underline{\sigma}^2 \tilde{\eta}_t\right),
\end{equation*}
By taking $\eta_t$ such that $\tilde{\eta}_t  = \frac{4}{\homo \underline{\sigma}^2((t+2) \log (t+2))} $, it holds that for $T>0$
\begin{equation*}
\begin{split}
    \bE_{\cA, \cD} \| \vw_{T} \|^2 \leq& \exp\left(- \frac{1}{2} \homo \underline{\sigma}^2 \tilde{\eta}_{T-1}\right) \bE_{\cA, \cD} \|\vw_{T-1} \|^2 \\
    \leq& \exp\left(- \frac{1}{2} \homo \underline{\sigma}^2 \sum_{t \in [T]} \tilde{\eta}_t\right) \bE_{\cA} \|\vw_{0} \|^2 \\
    \leq& \frac{\bE_{\cA} \|\vw_{0} \|^2}{(\log T)^2},
\end{split}
\end{equation*}
where we use the fact that $\sum_{t \in [T]} \frac{1}{(t+2) \log (t+2)} \leq \log(\log(T))$. 
This leads to the learning rate
\begin{equation*}
    \bE_{\cA, \cD} \eta_t =\bE_{\cA, \cD}  \frac{\tilde{\eta}_t}{\| \vw_t \|} \geq \frac{\tilde{\eta}_t}{[\bE_{\cA, \cD} \| \vw_t \|^2]^{1/2}} = \frac{4 \log t}{\sqrt{\bE_{\cA} \|\vw_{0} \|^2} \homo \underline{\sigma}^2 (t+2) \log (t+2)} = \Omega(1/t),
\end{equation*}
where we use Jensen's inequality, and the assumption that $\bE_{\cA} \|\vw_{0} \|^2$ is a constant. 

We next prove Lemma~\ref{lem: effective stepsize for 1/t} by contradiction.
Assume that the stepsize $\bE_{\cA, \cD} \eta_t = \cO(1/t)$ corresponds to a larger effective stepsize $\tilde{\eta}_t = \omega(\frac{1}{t\log t})$.
In this case, it still holds for the weight norm that $\bE_{\cA, \cD} \| \vw_{T} \|^2 = \cO( \frac{\bE_{\cA, \cD} \|\vw_{0} \|^2}{(\log T)^2})$.
Therefore, the corresponding learning rate would become larger according to the inequality that $\bE_{\cA} \eta_t \geq \frac{\tilde{\eta}_t}{[\bE_{\cA, \cD} \| \vw_t \|^2]^{1/2}} = \omega(1/t)$, which contradict with the statement that $\eta_t = \cO(1/t)$.
Therefore, given the stepsize $\eta_t = \cO(1/t)$, the corresponding effective learning rate must be $\tilde{\eta}_t = \cO(\frac{1}{t \log t})$.
The proof is done.
\end{proof}

\begin{lemma}
\label{lem: optimization}
Under the assumptions in Theorem~\ref{thm: separation} with smoothness (parameter $\beta$), PL condition (parameter $\mu$), strong growth condition (parameter $B$) and $\bE_I \hat{\cL}(\vv_t)  \ge (1+\alpha)/\alpha \hat{\cL}(\vv^*)$, given the effective learning rate $\tilde{\eta}_t \leq \frac{1}{\beta B^2}$, the convergence rate would be
\begin{equation*}
    \bE_{\cA, \cD}  \hat{\cL}(\vv_{T}) - \min_\vv \hat{\cL}(\vv) \leq \exp\left (-\lambda\sum \tilde{\eta}_t\right) \bE_{\cA, \cD} (\hat{\cL} (\vv_0)- \min_\vv \hat{\cL}(\vv) ).
\end{equation*}
Therefore, taking $\tilde{\eta}_t = c/t$ leads to 
\begin{equation*}
    \bE_{\cA, \cD} \hat{\cL}(\vv_{T}) - \min_\vv \hat{\cL}(\vv) \leq \frac{1}{T^{c\lambda}}  \bE_{\cA, \cD} (\hat{\cL} (\vv_0)-\min_\vv \hat{\cL}(\vv)  ).
\end{equation*}
Taking $\tilde{\eta}_t = c/[t \log t]$ leads to 
\begin{equation*}
   \bE_{\cA, \cD} \hat{\cL}(\vv_{T}) - \min_\vv \hat{\cL}(\vv)  \leq \frac{1}{[\log T]^{c\lambda}}  \bE_{\cA, \cD} (\hat{\cL} (\vv_0)- \min_\vv \hat{\cL}(\vv)  ).
\end{equation*}
\end{lemma}

\begin{proof}[Proof of Lemma~\ref{lem: optimization}]
For the ease of notation, denote $\vv^*$ as the one of the vectors which attains $\hat{\cL}(\vv^*) = \min_\vv \hat{\cL}(\vv)$. 
We first focus on the unconstrained iteration on the sphere, that is 
\begin{equation*}
    \bar{\vv}_{t+1} = \vv_{t} - \tilde{\eta}_t \nabla_{\vv_t} \ell_t(\vv_t). 
\end{equation*}
Notice that by strong growth condition, it holds that~(\citep{schmidt2013fastconvergencestochasticgradient}, (12)) when $\tilde{\eta}_t \leq \frac{1}{\beta B^2}$
\begin{equation}\label{eqn: temp1}
    \bE_{I} \hat{\cL} (\bar{\vv}_{t+1}) \leq \hat{\cL} (\vv_t) - \tilde{\eta}_t \left(1 - \frac{\tilde{\eta}_t \beta B^2}{2}\right) \| \nabla_{\vv_t} \hat{\cL}(\vv_t) \|^2 \leq \hat{\cL} (\vv_t) - \frac{\tilde{\eta}_t}{2}\| \nabla_{\vv_t} \hat{\cL}(\vv_t) \|^2.
\end{equation}
Besides, by PL condition and strong growth condition, for any $t$ and $\vv_t$, it holds that
\begin{equation*}
    \| \nabla_{\vv_t} \hat{\cL} (\vv_t)\|^2 \geq \frac{1}{B^2} \max_i \{ \|  \nabla_{\vv_t} \ell_it(\vv_t) \|^2 \} \geq \frac{2\mu}{B^2} (\ell_it (\vv_t) - \ell_it (\vv^*)).
\end{equation*}
By taking expectation $\bE_I$ over the last iteration, it holds that (note that $\vv_t$ is independent of the choice of the last iteration $\ell_t$)
\begin{equation}\label{eqn: temp2}
     \| \nabla_{\vv_t} \hat{\cL} (\vv_t)\|^2 \geq \frac{2\mu}{B^2} (\bE_I \ell_i (\vv_t) - \bE_I \ell_i (\vv^*)) \geq \frac{2\mu}{B^2} (\hat{\cL} (\vv_t) - \hat{\cL} (\vv^*)).
\end{equation}
Combining \eqref{eqn: temp1} and \eqref{eqn: temp2}leads to 
\begin{equation*}
    \bE_I \hat{\cL} (\bar{\vv}_{t+1}) \leq \left(1 - \frac{\tilde{\eta}_t \mu}{B^2}\right) \hat{\cL}(\vv_t) + \frac{\tilde{\eta}_t \mu}{B^2}\hat{\cL} (\vv^*)).
\end{equation*}
Next, we project $\bar{\vv}_{t+1}$ back onto the unit sphere. 
Analogous to Lemma~\ref{lem: norm ratio}, it holds that $\frac{\|\vw_t\|}{\|\vw_{t+1}\|} \le (1 - 2\homo\bar{\sigma}^2\tilde{\eta}_t)^{-\frac{1}{2}} \le 1 + 2\homo\bar{\sigma}^2\tilde{\eta}_t$.
Thus:
\begin{equation*}
\begin{split}
    \bE_I \hat{\cL} (\vv_{t+1}) - \hat{\cL} (\vv^*)
    &= \bE_I \left[\frac{\|\vw_t\|}{\|\vw_{t+1}\|}\right]^\homo\hat{\cL}(\bar{\vv}_{t+1}) - \hat{\cL} (\vv^*) \\
    &\le (1 + 4\homo^2\bar{\sigma}^2\tilde{\eta}_t)\left[\left(1 - \frac{\tilde{\eta}_t \mu}{B^2}\right) \bE_I\hat{\cL}(\vv_t) + \frac{\tilde{\eta}_t \mu}{B^2}\hat{\cL} (\vv^*)\right] - \hat{\cL} (\vv^*) \\
    & = (1 + 4\homo^2\bar{\sigma}^2\tilde{\eta}_t)\left(1 - \frac{\tilde{\eta}_t \mu}{B^2}\right) \left(\bE_I\hat{\cL}(\vv_t) - \hat{\cL} (\vv^*)\right) + 4\homo^2\bar{\sigma}^2\tilde{\eta}_t\hat{\cL} (\vv^*).
\end{split}
\end{equation*}
By $\bE_I \hat{\cL}(\vv_t)  \ge (1+\alpha)/\alpha \hat{\cL}(\vv^*)$, it holds that:
\begin{equation*}
\begin{split}
    \bE_I \hat{\cL} (\vv_{t+1}) - \hat{\cL} (\vv^*) 
    &\le (1 + 4\homo^2\bar{\sigma}^2\tilde{\eta}_t)\left(1 - \frac{\tilde{\eta}_t \mu}{B^2}\right) \left(\bE_I\hat{\cL}(\vv_t) - \hat{\cL} (\vv^*)\right) \\ 
    & \quad + 4\alpha\homo^2\bar{\sigma}^2\tilde{\eta}_t\left(\bE_I\hat{\cL}(\vv_t) - \hat{\cL}(\vv^*)\right) \\
    &\le \left[1 - \tilde{\eta}_t\left(\frac{\mu}{B^2} -4\homo^2\bar{\sigma}^2(1+\alpha)\right)\right]\left(\bE_I\hat{\cL} (\vv_{t+1}) - \hat{\cL}(\vv^*)\right).
\end{split}
\end{equation*}
Note that $\lambda \triangleq \frac{\mu}{B^2} -4\homo^2\bar{\sigma}^2(1+\alpha) > 0$.
Taking the expectation over the algorithm and dataset, the convergence would be 
\begin{equation*}
\begin{split}
   \bE_{\cA, \cD} \hat{\cL} (\vv_{t+1}) - \hat{\cL} (\vv^*) 
   \leq & \bE_{\cA, \cD}(\hat{\cL} (\vv_0)- \hat{\cL} (\vv^*) ) \prod (1 - \tilde{\eta}_t \lambda) \\
   \leq & \exp(-\lambda\sum \tilde{\eta}_t) \bE_{\cA, \cD}(\hat{\cL} (\vv_0)- \hat{\cL} (\vv^*) ).
\end{split}
\end{equation*}
\end{proof}

\subsection{Technical Lemmas}
\label{app:technical_lemmas}
\begin{lemma}[Weight Norm Iteration]
\label{lem: weight norm iteration}
    Under Assumption~\ref{assump: loss bound} and Assumption~\ref{assump: approximate smooth and lipshitz}, assume that $H>0$ and $\tilde{\eta}_t \leq \frac{\homo \underline{\sigma}^2}{2 L^2}$, if the loss is $H$-homogeneous, it holds that
    \begin{equation*}
        \bE_{\cA, \cD} \| \vw_{t+1} \|^2 \leq \bE_{\cA, \cD} \|\vw_{t} \|^2 ( 1 - \frac{1}{2} \homo \underline{\sigma}^2 \tilde{\eta}_t),
    \end{equation*}
    where the expectation $\bE_\cA$ is taken over the training algorithm.
\end{lemma}

\begin{proof}[Proof of Lemma~\ref{lem: weight norm iteration}]
Note that due to Lemma~\ref{lem: L-homogeneous property} (argument~3), it holds that
\begin{equation*}
\begin{split}
    \bE_I \| \vw_{t+1} \|^2 =& \bE_I \|\vw_{t} \|^2 ( 1 + \tilde{\eta}_t^2 \|\nabla_{\vv_t} \ell_t(\vv_t) \|^2 - 2 \homo \tilde{\eta}_t \ell_t(\vv_t)) \\
    =& \|\vw_{t} \|^2 ( 1 + \tilde{\eta}_t^2 \bE_I \|\nabla_{\vv_t} \ell_t(\vv_t) \|^2 - 2 \homo \tilde{\eta}_t \hat{\cL}(\vv_t)),
\end{split}
\end{equation*}
where the expectation $\bE_I$ is taken over the choice of the last iteration. 
Note that we use the fact that $\vv_t$ is independent of the choice of the last iteration (namely, $\ell_t$), and therefore, 
\begin{equation*}
   \bE_I \ell_t(\vv_t)=\frac{1}{n}\sum_{i\in[n]}\ell_t(\vv_t;X_i,Y_i)= \hat{\cL}_t(\vv_t).
\end{equation*}

We next plug the bound of Assumption~\ref{assump: loss bound}, that is, $\hat{\cL}(\vv_t) \geq \frac{1}{2} \underline{\sigma}^2$, and the bound of Assumption~\ref{assump: approximate smooth and lipshitz}, that is, $\sup_{\vv: \| \vv\|=1}\| \nabla_{\vv} \ell_t(\vv) \| \leq L$, namely,
\begin{equation*}
\begin{split}
    \bE_I \| \vw_{t+1} \|^2 
    \leq & \|\vw_{t} \|^2 ( 1 + \tilde{\eta}_t^2 L^2 - 2 \homo \tilde{\eta}_t \hat{\cL}(\vv_t)) \\
    \leq & \|\vw_{t} \|^2 ( 1 + \tilde{\eta}_t^2 L^2 -  \homo \tilde{\eta}_t \underline{\sigma}^2) \\
    \leq & \|\vw_{t} \|^2 ( 1 - \tilde{\eta}_t (\homo \underline{\sigma}^2- \tilde{\eta}_t L^2) ).
\end{split}
\end{equation*}
By the requirement that $\tilde{\eta}_t \leq \frac{\homo \underline{\sigma}^2}{2 L^2}$, it holds that
\begin{equation*}
\begin{split}
    \bE_I \| \vw_{t+1} \|^2 
    \leq & \|\vw_{t} \|^2 ( 1 - \frac{1}{2} \homo \underline{\sigma}^2 \tilde{\eta}_t).
\end{split}
\end{equation*}
Taking expectations over the whole algorithm and the dataset finishes the proof. 

\textbf{Remark: Extension to the non-Lipschitz setting.}
This result extends to the non-Lipschitz scenario. 
The analysis relies on the self-bounding property of $\beta$-smooth functions:
\begin{equation*}
    \begin{split}
    \bE_I \| \vw_{t+1} \|^2 
    \leq & \|\vw_{t} \|^2 ( 1 + 2\beta\tilde{\eta}_t^2 \bE_I[\ell_t(\vv_t)] - 2 \homo \tilde{\eta}_t \hat{\cL}(\vv_t)) \\
    = & \|\vw_{t} \|^2 ( 1 + 2\beta\tilde{\eta}_t^2 \hat{\cL}(\vv_t) - 2 \homo \tilde{\eta}_t \hat{\cL}(\vv_t)) \\
    = & \|\vw_{t} \|^2 ( 1 - 2\tilde{\eta}_t \hat{\cL}(\vv_t) (\homo - \beta\tilde{\eta}_t)) \\
    \leq & \|\vw_{t} \|^2 ( 1 - \underline{\sigma}^2\tilde{\eta}_t (\homo - \beta\tilde{\eta}_t) ).
\end{split}
\end{equation*}
Consequently, if $\tilde{\eta}_t \le \frac{\homo}{2\beta}$, we still obtain:
\begin{equation*}
\begin{split}
    \bE_I \| \vw_{t+1} \|^2 
    \leq & \|\vw_{t} \|^2 ( 1 - \frac{1}{2} \homo \underline{\sigma}^2 \tilde{\eta}_t).
\end{split}
\end{equation*}
\end{proof}

\subsection{Proofs for Section~\ref{sec:relax-lipschitz}}
\label{app:proof_refined}
This section provides the proof for the results presented in Section~\ref{sec:relax-lipschitz}. 
We first derive Lemma~\ref{thm:gen_decomposition}, and subsequently establish Theorem~\ref{thm:gen_without_lipschitz}. 

\subsubsection{Proof of Lemma~\ref{thm:gen_decomposition}}
\label{subsec:proof_thm_5_1}
\begin{proof}
    This proof decomposes the generalization gap by utilizing the notion of \emph{conditional on-average stability} (Definition~\ref{def: stability}).
    
    Let $z_i \triangleq (\X_i, \Y_i)$ and define the individual loss difference on the $i$-th sample as $\delta(\vv_T, \vv_T^{(i)}) \triangleq |\ell(\vv_T; z_i) - \ell(\vv_T^{(i)}; z_i)|$. 
    Following the standard stability analysis in \citet{DBLP:journals/jmlr/BousquetE02}, we have
    \begin{align*}
        \bE_{\cA,\cD} [\cL^s(\vw_T) - \hat{\cL}^s(\vw_T)] &= \bE_{\cA, \cD}[\cL(\vv_T)-\hat{\cL}(\vv_T)] \\
        &= \bE_{\cA, \cD}[\frac{1}{n}\sum_{i=1}^{n}\ell(\vv_T^{(i)}; z_i) - \frac{1}{n}\sum_{i=1}^{n}\ell(\vv_T; z_i)] \\
        &\le \frac{1}{n} \sum_{i=1}^n \bE_{\cA, \cD} [\delta(\vv_T, \vv_T^{(i)})].
    \end{align*}
    Let $\mathcal{E}_{(i)}$ denote the event that the algorithm $\cA$ does \emph{not} select the index $i$ within the first $t_0$ iterations. 
    The probability of the complement event $\mathcal{E}_{(i)}^c$ (i.e., index $i$ is selected at least once) is bounded by:
    \begin{equation*}
        P(\mathcal{E}_{(i)}^c) \le \sum_{t=1}^{t_0} \frac{1}{n} = \frac{t_0}{n},
    \end{equation*}
    Analogous to the technique in \eqref{eq:technique}, invoking the Bounded Loss assumption (Assumption~\ref{assump: loss bound}) yields:
    \begin{align}
    \label{eq:decomp_step2}
        \bE [\delta(\vv_T, \vv_T^{(i)})] 
        &= P(\mathcal{E}_{(i)}^c) \bE [\delta(\vv_T, \vv_T^{(i)}) \mid \mathcal{E}_{(i)}^c] + P(\mathcal{E}_{(i)}) \bE [\delta(\vv_T, \vv_T^{(i)}) \mid \mathcal{E}_{(i)}] \nonumber \\
        &\le \frac{t_0}{n} \cdot \bar{\sigma}^2 + \bE [\delta(\vv_T, \vv_T^{(i)}) \mid \mathcal{E}_{(i)}]. 
    \end{align}
    We recall the result derived by \citet{DBLP:conf/icml/LeiY20}, as in Proposition~\ref{prop:lei-gen-bound-appendix}. 
    Notably, its derivation relies solely on the properties of the individual loss rather than the empirical loss. 
    Consequently, under the assumption of $\beta$-smooth, the following inequality holds:
    \begin{equation}
    \label{eq:diff}
        |\ell(\vv_T; z) - \ell(\vv_T^{(i)}; z)| \le \frac{\beta}{\zeta}\ell(\vv_T; z) + \frac{\beta+\zeta}{2}\|\vv_T-\vv_T^{(i)}\|^2. 
    \end{equation}
    where $\zeta > 0$ is a free parameter. 
    Conditioning \eqref{eq:diff} on the event $\mathcal{E}_{(i)}$ and taking the expectation over the randomness of the algorithm and datasets on both sides, we obtain:
    \begin{equation*}
        \bE [\delta(\vv_T, \vv_T^{(i)}) \mid \mathcal{E}_{(i)}] 
        \le \frac{\beta}{\zeta} \bE [\ell(\vv_T; z_i) \mid \mathcal{E}_{(i)}] + \frac{\beta+\zeta}{2} \bE [\|\vv_T - \vv_T^{(i)}\|^2 \mid \mathcal{E}_{(i)}]. 
    \end{equation*}
    The first term $\bE [\ell(\vv_T; z_i) \mid \mathcal{E}_{(i)}]$ is uniformly bounded by $\bar{\sigma}^2$ under Assumption~\ref{assump: loss bound}, while the second term corresponds exactly to the definition of the conditional on-average stability. 
    Substituting this back into \eqref{eq:decomp_step2} yields:
    \begin{equation*}
        \bE [\delta(\vv_T, \vv_T^{(i)})] \le \frac{t_0}{n}\bar{\sigma}^2 + \frac{\beta}{\zeta}\bar{\sigma}^2 + \frac{\beta+\zeta}{2} \bE [\|\vv_T - \vv_T^{(i)}\|^2 \mid \mathcal{E}_{(i)}].
    \end{equation*}
    Finally, by setting a uniform $t_0$ for all indices $i$ and averaging the individual bounds over $i=1, \dots, n$, the generalization gap satisfies:
    \begin{equation}
    \label{eq:general_bound}
         \begin{split}
        \bE_{\cA,\cD} [\cL^s(\vw_T) - \hat{\cL}^s(\vw_T)] 
        &\le \frac{1}{n} \sum_{i=1}^n \left[ \left(\frac{t_0}{n} + \frac{\beta}{\zeta}\right)\bar{\sigma}^2 + \frac{\beta+\zeta}{2} \bE_{\cA, \cD} [\|\vv_T - \vv_T^{(i)}\|^2 \mid \mathcal{E}_{(i)}] \right] \nonumber \\
        &= \left(\frac{t_0}{n} + \frac{\beta}{\zeta}\right)\bar{\sigma}^2 + \frac{\beta+\zeta}{2} \underbrace{\frac{1}{n} \sum_{i=1}^n \bE_{\cA, \cD} [\|\vv_T - \vv_T^{(i)}\|^2 \mid \mathcal{E}_{(i)}]}_{\epsilon^2_{\text{avg}, T \mid t_0}}.
        \end{split}
    \end{equation}
    This concludes the proof.
\end{proof}

\subsubsection{Proof of Theorem~\ref{thm:gen_without_lipschitz}}
\label{subsec:proof_thm_5_2}
We begin by establishing a recursive bound for stability in the following lemma. 
\begin{lemma}
\label{lem:stability_recurrence_homo}
Suppose the assumptions of Theorem~\ref{thm:gen_without_lipschitz} hold. 
Using the effective stepsize $\tilde{\eta}_t$ defined in Lemma~\ref{lem: generalization error initial}, for any iteration $t > t_0$, the following recurrence holds:
\begin{equation*}
    \bE_I \|\vv_{t+1} - \vv_{t+1}^{(i)}\|^2 \le \exp\left(\frac{m'_1}{t+2} + \frac{1}{T}\right) \|\vv_t - \vv_t^{(i)}\|^2 + m'_2\left(1+\frac{T}{n}\right)\frac{1}{n}\frac{1}{(t+2)^2},
\end{equation*}
where $m^\prime_1 = 8\frac{c_2\bar{\sigma}^2}{\underline{\sigma}^2} + \frac{4c_2\beta}{\homo \underline{\sigma}^2}$ and $m^\prime_2 = \frac{32 e c_2^2 \beta\bar{\sigma}^2}{n \homo^2 \underline{\sigma}^4}$.
\end{lemma}
\begin{proof}
We analyze the update at step $t > t_0$. 
Let $i_t$ denote the sample index selected by SGD at step $t$.

With probability $1 - 1/n$, the algorithm selects an index $i_t \neq i$ (i.e., the identical sample is used).
By the $\beta$-smoothness of the loss function, the update satisfies: 
\begin{equation}
    \label{eq:same}
    \begin{aligned}
    \|\bar{\vv}_{t+1} - \bar{\vv}_{t+1}^{(i)}\|^2 
    &= \|(\vv_t - \vv_t^{(i)}) - \tilde{\eta}_t [\nabla_{\vv_t}\ell_t(\vv_t))- \nabla_{\vv_t^{(i)}} \ell_t(\vv_t^{(i)})]\|^2 \\
    &\le [\|\vv_t - \vv_t^{(i)}\| + \tilde{\eta}_t\|\nabla_{\vv_t}\ell_t(\vv_t)) - \nabla_{\vv_t^{(i)}}\ell_t(\vv_t^{(i)})\|]^2\\
    &\le (1 + \beta \tilde{\eta}_t)^2 \|\vv_t - \vv_t^{(i)}\|^2.
    \end{aligned}
\end{equation}
With probability $1/n$, the algorithm selects the index $i_t = i$ (i.e, the samples differ). 
Using the inequality $\|a+b\|^2 \le (1+p)\|a\|^2 + (1+1/p)\|b\|^2$ with $p > 0$, the self-bounding property $\|\nabla \ell(\vv)\|^2 \le 2\beta \ell(\vv)$, and the Bounded Loss assumption, we have
\begin{equation}
    \label{eq:diff_2}
    \begin{aligned}
     \|\bar{\vv}_{t+1} - \bar{\vv}_{t+1}^{(i)}\| 
     &= \|(\vv_t - \vv_t^{(i)}) - \tilde{\eta}_t [\nabla_{\vv_t}\ell_t(\vv_t))- \nabla_{\vv_t^{(i)}} \ell_t(\vv_t^{(i)})]\|^2\\
     &\le (1+p) \|\vv_t - \vv_t^{(i)}\|^2  + \tilde{\eta}_t^2 (1+1/p) \|\nabla_{\vv_t}\ell_t(\vv_t)- \nabla_{\vv_t^{(i)}} \ell_t(\vv_t^{(i)})\|^2\\
     &\le (1+p) \|\vv_t - \vv_t^{(i)}\|^2  + 2\tilde{\eta}_t^2 (1+1/p) [\|\nabla_{\vv_t}\ell_t(\vv_t)\|^2 + \|\nabla_{\vv_t^{(i)}} \ell_t(\vv_t^{(i)})\|^2]\\
     &\le (1+p) \|\vv_t - \vv_t^{(i)}\|^2  + 4\beta(1+1/p)\tilde{\eta}_t^2 [\ell(\vv_t; z_i) + \ell(\vv_t^{(i)}; z_i')] \\
     &\le (1+p) \|\vv_t - \vv_t^{(i)}\|^2  + 8\beta(1+1/p)\tilde{\eta}_t^2 \bar{\sigma}^2. 
    \end{aligned}
\end{equation}
Taking the expectation over the choice of $i_t$, we combine \eqref{eq:same} and \eqref{eq:diff_2}:
\begin{equation}
\label{eq:until_bar}
    \bE_I \|\bar{\vv}_{t+1} - \bar{\vv}_{t+1}^{(i)}\|^2  \le \left[\left(1-\frac{1}{n}\right)(1 + \beta\tilde{\eta}_t)^2 + \frac{1}{n}(1+p)\right]  \|{\vv}_{t} - {\vv}_{t}^{(i)}\|^2 + \frac{8\beta}{n}\left(1+\frac{1}{p}\right)\tilde{\eta}_t^2 \bar{\sigma}^2.
\end{equation}
Utilizing the norm ratio bound in \eqref{eq:norm ratio bound}, we have 
\begin{equation}
\label{eq:full}
    \bE_I \|\vv_{t+1} - \vv_{t+1}^{(i)}\|^2 \le (1 + 4\frac{c_2 \bar{\sigma}^2}{\underline{\sigma}^2}\frac{1}{t+2})^2 \bE_I \|\bar{\vv}_{t+1} - \bar{\vv}_{t+1}^{(i)}\|^2. 
\end{equation}
Setting $p = n/T$ and substituting \eqref{eq:until_bar} into \eqref{eq:full}, we obtain:
{
\small
\begin{equation}
    \label{eq:terms}
    \begin{aligned}
    & \quad \bE_I \|\vv_{t+1} - \vv_{t+1}^{(i)}\|^2 \\
    &\le (1 + 4\frac{c_2 \bar{\sigma}^2}{\underline{\sigma}^2}\frac{1}{t+2})^2 \left\{\left[\left(1-\frac{1}{n}\right)(1 + \beta\tilde{\eta}_t)^2 + \frac{1}{n}(1+p)\right]  \|{\vv}_{t} - {\vv}_{t}^{(i)}\|^2 + \frac{8\beta}{n}\left(1+\frac{1}{p}\right)\tilde{\eta}_t^2 \bar{\sigma}^2\right\} \\
    &=\left\{\left(1 + 4\frac{c_2 \bar{\sigma}^2} {\underline{\sigma}^2}\frac{1}{t+2}\right)^2 \left[\left(1-\frac{1}{n}\right)(1 + \beta\tilde{\eta}_t)^2 + \frac{1}{n}\left(1+\frac{n}{T}\right)\right]\right\}\|{\vv}_{t} - {\vv}_{t}^{(i)}\|^2 \\
    & \quad +\left(1 + 4\frac{c_2 \bar{\sigma}^2}{\underline{\sigma}^2}\frac{1}{t+2}\right)^2 \frac{8\beta}{n}\left(1+\frac{1}{p}\right)\tilde{\eta}_t^2 \bar{\sigma}^2.
\end{aligned}
\end{equation}
}
For the first term in \eqref{eq:terms}, we have:
\begin{align*}
    &\quad \left(1 + 4\frac{c_2 \bar{\sigma}^2} {\underline{\sigma}^2}\frac{1}{t+2}\right)^2 \left[\left(1-\frac{1}{n}\right)(1 + \beta\tilde{\eta}_t)^2 + \frac{1}{n}\left(1+\frac{n}{T}\right)\right] \\
    &\le \left(1 + 4\frac{c_2 \bar{\sigma}^2}{\underline{\sigma}^2}\frac{1}{t+2}\right)^2 \left[ (1 + \beta\tilde{\eta}_t)^2 + \frac{1}{T} \right] \\
    &\le \left(1 + 4\frac{c_2 \bar{\sigma}^2}{\underline{\sigma}^2}\frac{1}{t+2}\right)^2 (1 + \beta\tilde{\eta}_t)^2 \left( 1 + \frac{1}{T} \right) \\
    &\le \exp\left(8\frac{c_2 \bar{\sigma}^2}{\underline{\sigma}^2}\frac{1}{t+2}\right) \exp(2\beta\tilde{\eta}_t) \exp\left(\frac{1}{T}\right) \\
    &= \exp\left( \frac{1}{T} + 8\frac{c_2 \bar{\sigma}^2}{\underline{\sigma}^2}\frac{1}{t+2} + 2\beta \frac{2c_2}{H\underline{\sigma}^2 (t+2)} \right) \\
    &= \exp\left( \frac{1}{T} + \left[ 8\frac{c_2 \bar{\sigma}^2}{\underline{\sigma}^2} + \frac{4c_2\beta}{H\underline{\sigma}^2} \right] \frac{1}{t+2} \right) = \exp\left( \frac{1}{T} + \frac{m'_1}{t+2} \right).
\end{align*}
For the second term in \eqref{eq:terms}, we have:
\begin{align*}
    &\quad \left(1 + 4\frac{c_2 \bar{\sigma}^2}{\underline{\sigma}^2}\frac{1}{t+2}\right)^2 \frac{8\beta}{n}\left(1+\frac{1}{p}\right)\tilde{\eta}_t^2 \bar{\sigma}^2 \\
    &\le \exp\left(8\frac{c_2 \bar{\sigma}^2}{\underline{\sigma}^2}\frac{1}{t+2}\right) \frac{8\beta}{n} \left(1+\frac{T}{n}\right) \left(\frac{2 c_2}{\homo\underline{\sigma}^2(t+2)}\right)^2 \bar{\sigma}^2 \\
    &\le e \cdot \frac{8\beta}{n} \left(1+\frac{T}{n}\right) \frac{4c_2^2}{\homo^2\underline{\sigma}^4 (t+2)^2} \bar{\sigma}^2 \\
    &= \left[ \frac{32 e c_2^2 \beta \bar{\sigma}^2}{ H^2 \underline{\sigma}^4} \right] \left(1+\frac{T}{n}\right) \frac{1}{(t+2)^2} \frac{1}{n}  \\
    &= m'_2 \left(1+\frac{T}{n}\right) \frac{1}{n} \frac{1}{(t+2)^2}.
\end{align*}
Consequently, \eqref{eq:terms} is equivalent to
\begin{equation}
\label{eq:single_stability}
    \bE_I \|\vv_{t+1} - \vv_{t+1}^{(i)}\|^2 \le \exp\left(\frac{m'_1}{t+2} + \frac{1}{T}\right) \|\vv_t - \vv_t^{(i)}\|^2 + m'_2\left(1+\frac{T}{n}\right)\frac{1}{n}\frac{1}{(t+2)^2}. 
\end{equation}
This completes the proof of Lemma~\ref{lem:stability_recurrence_homo}. 
\end{proof}

We then proceed to derive the generalization bound in Theorem~\ref{thm:gen_without_lipschitz}. 
We first bound the conditional on-average stability, and then optimize the resulting generalization gap.  
\begin{proof}
    Taking expectations of \eqref{eq:single_stability} over the randomness of the algorithm and datasets, conditioned on $\mathcal{E}_{(i)}$ , we obtain:
    \begin{equation*}
    \begin{aligned}
        &\bE_{\cA, \cD} [\|\vv_{t+1} - \vv_{t+1}^{(i)}\|^2 \mid \mathcal{E}_{(i)}] \\
        \le &\exp\left(\frac{m'_1}{t+2} + \frac{1}{T}\right) \bE_{\cA, \cD}[\|\vv_t - \vv_t^{(i)}\|^2\mid \mathcal{E}_{(i)}] + m'_2\left(1+\frac{T}{n}\right)\frac{1}{n}\frac{1}{(t+2)^2}. 
    \end{aligned}
\end{equation*}
Averaging over all samples $1, \dots, n$, and defining $\Delta_t \triangleq \frac{1}{n}\sum_{i=1}^{n}\bE_{\cA, \cD}[\|\vv_t - \vv_t^{(i)}\|^2 \mid \mathcal{E}_{(i)}]$, we have:
\begin{equation*}
    \Delta_{t+1} \le \exp\left(\frac{m'_1}{t+2} + \frac{1}{T}\right) \Delta_t + m'_2\left(1+\frac{T}{n}\right)\frac{1}{n}\frac{1}{(t+2)^2}.
\end{equation*}
Let $K_{n, T} \triangleq \frac{e}{m^\prime_1+1} (1 + \frac{T}{n})\frac{1}{n}m^\prime_2$, where $e$ is the base of the natural logarithm. 
Then:
\begin{equation*}
    \Delta_{t+1} \le \exp\left(\frac{m'_1}{t+2} + \frac{1}{T}\right) \Delta_t + \frac{m^\prime_1+1}{e}K_{n, T}\frac{1}{(t+2)^2}.
\end{equation*}
Unrolling the recurrence from $t = t_0$ to $T-1$ yields:
\begin{equation*}
    \begin{split}
        \Delta_T 
        &\le \frac{m^\prime_1+1}{e} K_{n, T} \sum_{t=t_0}^{T-1} \frac{1}{(t+2)^2} \prod_{k=t+1}^{T-1} \exp\left(\frac{m'_1}{k+2} + \frac{1}{T}\right) \\
        &= \frac{m^\prime_1+1}{e} K_{n, T} \sum_{t=t_0}^{T-1} \frac{1}{(t+2)^2} \exp \sum_{k=t+1}^{T-1}\left(\frac{m'_1}{k+2} + \frac{1}{T}\right) \\
        &= \frac{m^\prime_1+1}{e} K_{n, T} \sum_{t=t_0}^{T-1} \frac{1}{(t+2)^2} \exp\left( m'_1\sum_{k=t+1}^{T-1}\frac{1}{k+2} + \frac{T-t-1}{T}\right) \\
        &\le (m^\prime_1+1) K_{n, T} \sum_{t=t_0}^{T-1} \frac{1}{(t+2)^2} \exp\left( m'_1\sum_{k=t+1}^{T-1}\frac{1}{k+2}\right).
    \end{split}
\end{equation*}
Using the inequality $\sum_{k=t+1}^{T-1}\frac{1}{k+2} \le \ln(\frac{T}{t+1})$ and $\sum_{t=t_0}^{T-1} (t+1)^{-(2+m^\prime_1)} \le \int_{t_0}^{\infty} x^{-(2+m'_1)} dx = \frac{1}{1+m^\prime_1}t_0^{-(m^\prime_1+1)}$, we obtain:
\begin{equation}
\label{eq:stability_bound}
    \Delta_T \le (m^\prime_1+1) K_{n, T} T^{m^\prime_1}\sum_{t=t_0}^{T-1} (t+1)^{-(2+m^\prime_1)} \le K_{n, T} T^{m^\prime_1}t_0^{-(m^\prime_1+1)}.
\end{equation}

Substituting \eqref{eq:stability_bound} into \eqref{eq:general_bound} gives
\begin{equation*}
\begin{split}
    \bE_{\cA,\cD} [\cL^s(\vw_T) - \hat{\cL}^s(\vw_T)] 
    &\le \left(\frac{t_0}{n} + \frac{\beta}{\zeta}\right)\bar{\sigma}^2 + \frac{\beta+\zeta}{2} \Delta_T \\
    &\le \left(\frac{t_0}{n} + \frac{\beta}{\zeta}\right)\bar{\sigma}^2 + \frac{\beta+\zeta}{2}K_{n, T} T^{m^\prime_1}t_0^{-(m^\prime_1+1)} \\
    &= \frac{K_{n, T} T^{m^\prime_1}}{2} (\zeta + \beta) \cdot t_0^{-(m^\prime_1+1)} + \beta \bar{\sigma}^2 \frac{1}{\zeta} + \frac{\bar{\sigma}^2}{n}t_0.
\end{split}
\end{equation*}
Define $g(\zeta, t_0) \triangleq \frac{K_{n, T} T^{m^\prime_1}}{2} (\zeta + \beta) \cdot t_0^{-(m^\prime_1+1)} + \beta \bar{\sigma}^2 \frac{1}{\zeta} + \frac{\bar{\sigma}^2}{n}t_0$. 
Observing that the generalization gap is independent of $\zeta$ and $t_0$, we minimize $g$ to obtain:
\begin{equation*} \begin{split}
    &\bE_{\cA,\cD} [\cL^s(\vw_T) - \hat{\cL}^s(\vw_T)] \\
    \le &\min_{t_0, \zeta} g(t_0, \zeta) \\
    = &(m^\prime_1+3) / (m^\prime_1+1)^{\frac{m^\prime_1+2}{m^\prime_1+3}} \cdot\left( \beta e m^\prime_2 / 2 \right)^{\frac{1}{m^\prime_1+3}} \bar{\sigma}^{\frac{2m^\prime_1+4}{m^\prime_1+3}} \cdot n^{-\frac{m^\prime_1+2}{m^\prime_1+3}} \left( 1 + \frac{T}{n} \right)^{\frac{1}{m^\prime_1+3}} T^{\frac{m^\prime_1}{m^\prime_1+3}} \\
    = &\cO(n^{-\frac{m^\prime_1+2}{m^\prime_1+3}} \left( 1 + \frac{T}{n} \right)^{\frac{1}{m^\prime_1+3}} T^{\frac{m^\prime_1}{m^\prime_1+3}}).
\end{split}
\end{equation*}
This concludes the proof of Theorem~\ref{thm:gen_without_lipschitz}. 
\end{proof}

\subsection{Review of Several Stability Metrics}

\subsubsection{Review of Algorithmic Stability}
\label{app:stability-review}
Algorithmic stability is a popular technique in generalization theory, offering the following favorable properties:
\begin{proposition}\textnormal{\textbf{(Algorithmic stability and generalization)}}\label{prop: stability-gen}The arguments are based on~\citet{DBLP:journals/jmlr/BousquetE02,shalev2010learnability,DBLP:conf/icml/HardtRS16}.
    Assume that the algorithm is $\epsilon_T$-stable, meaning that for all datasets $S$ and $S^\prime$ differing in at most one sample, it holds that
    \begin{equation*}
        \sup_{(\X,\Y)} \bE_{\cA, \cD} [\ell(\vw_T; \X, \Y) - \ell(\vw_T^\prime; \X, \Y)] \leq \epsilon_T,
    \end{equation*}
    where $\vw_T$ and $\vw_T^\prime$ are trained on datasets $S$ and $S^\prime$ over $T$ iterations, respectively, and the expectation is taken over both the training set randomness and the algorithmic randomness. 
    Then the generalization gap can be bounded via algorithmic stability:
    \begin{equation*}
        \bE_{\cA,\cD} \cL(\vw_T) - \hat{\cL}(\vw_T) \leq \epsilon_T.
    \end{equation*}
    Furthermore, it holds that for SGD with replacement using stepsize $\eta_t$ for each iteration $t$:
    \begin{enumerate}
        \item For \textbf{convex} loss functions, provided they are $L$-Lipschitz and $\beta$-smooth, and the stepsize satisfies $\eta_T\leq 2/\beta$, it holds that $\epsilon_T \leq \frac{2 L^2}{n} \sum_{t=1}^T \eta_t$.
        \item For \textbf{non-convex} loss functions, provided they are $L$-Lipschitz and $\beta$-smooth, and the stepsize is non-increasing with $\eta_t \leq c/t$ for a given constant $c$, it holds that $\epsilon_T \leq \frac{1 + 1/\beta c}{n-1} (2cL^2)^{\frac{1}{\beta c+1}} T^\frac{\beta c}{\beta c + 1}$.
    \end{enumerate}
\end{proposition}

\subsubsection{Review of On-Average Model Stability}
\label{app:on-avg-background}
In this section, we provide the necessary background on the \emph{on-average model stability} introduced by~\citet{DBLP:conf/icml/LeiY20}. 
Unlike uniform stability, which requires an upper bound on the loss difference for any pair of sample in the distribution, on-average model stability measures the average distance of weights trained via the same algorithm on two datasets differing at most one sample. 

\begin{definition}[$\ell_2$ On-Average Model Stability]
Let $S = \{z_1, \dots, z_n\}$ and $S^{(i)} = \{z_1, \dots, z_i', \dots, \allowbreak z_n\}$ be two datasets differing only in the $i$-th sample. A randomized algorithm $\cA$ is $\epsilon_{\text{avg}}$-stable in $\ell_2$ on-average if:
\begin{equation*}
    \mathbb{E}_{S, \tilde{S}, \cA}\left[\frac{1}{n}\sum_{i=1}^{n}\|\cA(S) - \cA(S^{(i)})\|_2^2\right] \leq \epsilon_{\text{avg}}^2.
\end{equation*}
\end{definition}

Based on this definition, the generalization gap can be bounded as follows:

\begin{proposition}[Generalization via Model Stability]
\label{prop:lei-gen-bound-appendix}
If for any $z$, the loss $\ell(\vw; z)$ is non-negative and $\beta$-smooth, then for any free parameter $\zeta > 0$:
\begin{equation*}
    \bE_{\cA,\cD} [\cL(\vw_T) - \hat{\cL}(\vw_T)] \leq \frac{\beta}{\zeta}\bE_{\cA,\cD}[\hat{\cL}(\vw_T)] + \frac{\beta + \zeta}{2}\epsilon_{\text{avg}}^2,
\end{equation*}
where $\epsilon_{\text{avg}}^2$ denotes the on-average model stability of the algorithm’s output $\vw_T$.
\end{proposition}
\end{document}